\documentclass[10pt]{article} 
\usepackage[accepted]{tmlr}

\usepackage{graphicx}
\usepackage{url}
\usepackage{pifont}
\usepackage{amssymb}
\usepackage{algpseudocode}

\input{preamble.tex}
\usepackage[mathscr]{euscript}

\newcommand\tldrDone[1]{}

\usepackage{graphicx}
\usepackage{subcaption}

\newcommand{\R}{\mathbb{R}}
\newcommand{\N}{\mathbb{N}}

\newcommand{\eps}{\varepsilon}

\newcommand{\E}[2][]{\mathbb{E}_{\ifx &#1& \else #1 \fi}\left[#2\right]}

\newcommand{\var}{\mathsf{Var}}

\newcommand{\Ep}{\mathbb{E}}

\newcommand{\norm}[1]{\left\|{#1}\right\|} 

\newcommand{\est}[1]{\widehat{#1}}

\newcommand{\mc}[1]{\mathcal{#1}}

\newcommand{\defeq}{:=}

\definecolor{innerboxcolor}{rgb}{.9,.95,1}
\definecolor{outerlinecolor}{rgb}{.6,0,.2}
\definecolor{shcolor}{RGB}{27, 87, 14}
\definecolor{rckcolor}{RGB}{0,0,255}

\newcommand{\txt}[1]{\textup{#1}}
\newcommand{\iid}{\overset{\text{i.i.d.}}{\sim}}

\newcommand{\permtest}{\hyperref[algorithm:permtest]{\mathtt{PERMTEST}}}
\newcommand{\speartest}{\hyperref[algorithm:speartest]{\mathtt{MATCH}}}
\newcommand{\spearman}{\hyperref[algorithm:spearman-pvalue]{\mathtt{SPEARMAN}}}
\newcommand{\match}{\mathtt{LAP}}
\newcommand{\fisher}{\hyperref[algorithm:fisher]{\mathtt{FISHER}}}

\newcommand{\firstparam}{\theta_1}
\newcommand{\secondparam}{\theta_2}
\newcommand{\firstinit}{\theta_1^0}

\newcommand{\init}{\theta^0}
\newcommand{\firstP}{P_1}

\newcommand{\fmlp}{f_\txt{mlp}}
\newcommand{\flm}{f_\txt{LM}}

\newcommand{\gateproj}{G}
\newcommand{\upproj}{U}
\newcommand{\downproj}{D}
\newcommand{\act}{\sigma}

\newcommand{\Thetamlp}{\Theta_\txt{mlp}}
\newcommand{\Thetablock}{\Theta_\txt{block}}
\newcommand{\Thetalm}{\Theta_\txt{LM}}
\newcommand{\Thetain}{\Theta_\txt{in}}
\newcommand{\Thetaout}{\Theta_\txt{out}}
\newcommand{\Thetapre}{\Theta_\txt{attn}}

\newcommand{\thetamlp}{\theta_\txt{mlp}}
\newcommand{\thetablock}{\theta_\txt{block}}

\newcommand{\thetain}{\theta_\txt{in}}
\newcommand{\thetaout}{\theta_\txt{out}}
\newcommand{\thetapre}{\theta_\txt{attn}}

\newcommand{\fpre}{f_\txt{attn}}
\newcommand{\fpost}{f_\txt{post}}
\newcommand{\fin}{f_\txt{in}}
\newcommand{\fout}{f_\txt{out}}
\newcommand{\frest}{f_\txt{rest}}

\newcommand{\ltwo}{\phi_{\ell_2}}
\newcommand{\csw}{\phi_{U^{(\ell)}}}
\newcommand{\csh}{\phi_{H^{(\ell)}}}
\newcommand{\dcos}{\mathtt{cossim}}
\newcommand{\jsd}{\phi_\txt{JSD}}
\newcommand{\rob}{\phi_\txt{MATCH}}

\newcommand{\firstM}{W_1}

\newcommand{\secondM}{W_2}

\newcommand{\activation}[2]{X_{#1}^{(#2)}}

\newcommand{\constr}{\txt{constrained}}
\newcommand{\unconstr}{\txt{unconstrained}}

\title{Independence Tests for Language Models}

\newcommand\blfootnote[1]{%
  \begingroup
  \renewcommand\thefootnote{}\footnote{#1}%
  \addtocounter{footnote}{-1}%
  \endgroup
}

\begin{document}

\author{\name Sally Zhu*, Ahmed Ahmed*, Rohith Kuditipudi*, Percy Liang \\ \addr Department of Computer Science, Stanford University
}


\maketitle

\blfootnote{$\ast$ Equal contribution; more junior authors listed earlier. We share code at \url{https://github.com/ahmeda14960/model-tracing}.}

\begin{abstract}
    We consider the following problem: given the weights of two models, can we test whether they were trained independently---i.e., from independent random initializations? We consider two settings: \textit{constrained} and 
    \textit{unconstrained}.
    In the $\constr$ setting, we make assumptions about model architecture and training and propose a family of statistical tests that yield exact p-values with respect to the null hypothesis that the models are trained from independent random initializations. 
    These p-values are valid regardless of the composition of either model's training data; we compute them by simulating exchangeable copies of each model under our assumptions
    and comparing various similarity measures of weights and activations between the original two models versus these copies. 
    We report the p-values from these tests on pairs of 21 open-weight models (210 total pairs) and find we correctly identify all pairs of non-independent models. Notably, our tests remain effective even if one of the models was fine-tuned for many tokens. 
    In the $\unconstr$ setting, where we make no assumptions about training procedures, can change model architecture, and allow for adversarial evasion attacks, the previous tests no longer work.
    Instead, we propose a new test which matches hidden activations between two models, and use it to construct a test
    that is robust to adversarial transformations and to changes in model architecture. The test can also perform \textit{localized testing}: identifying specific non-independent components of models. 
    Though we no longer obtain exact p-values from this test, empirically we find it behaves as one and reliably distinguishes non-independent models. Notably, we can use the test to identify specific parts of one model that are derived from another (e.g., how Llama 3.1-8B was pruned to initialize Llama 3.2-3B, or shared layers between Mistral-7B and StripedHyena-7B),
    and it is even robust to retraining individual layers of either model from scratch.
\end{abstract}

\section{Introduction}

Consider the ways in which two models could be related: one model may be a finetune of the other; one could be spliced and pruned from certain parts of the other; both models could be separately fine-tuned from a common ancestor; finally, they could be independently trained from each other. 
We consider the problem of determining whether two models are independently trained versus not from their weights, which we formalize as a hypothesis testing problem in which the null hypothesis is that the weights of the two models are independent. We concretely treat only the weight initialization as random and thus consider two models with different random initial seeds as independent, even if both models were trained on the same data, or one model was distilled from the outputs of the other.

\begin{figure}[t]
    \centering
    \includegraphics[width=0.98\linewidth]{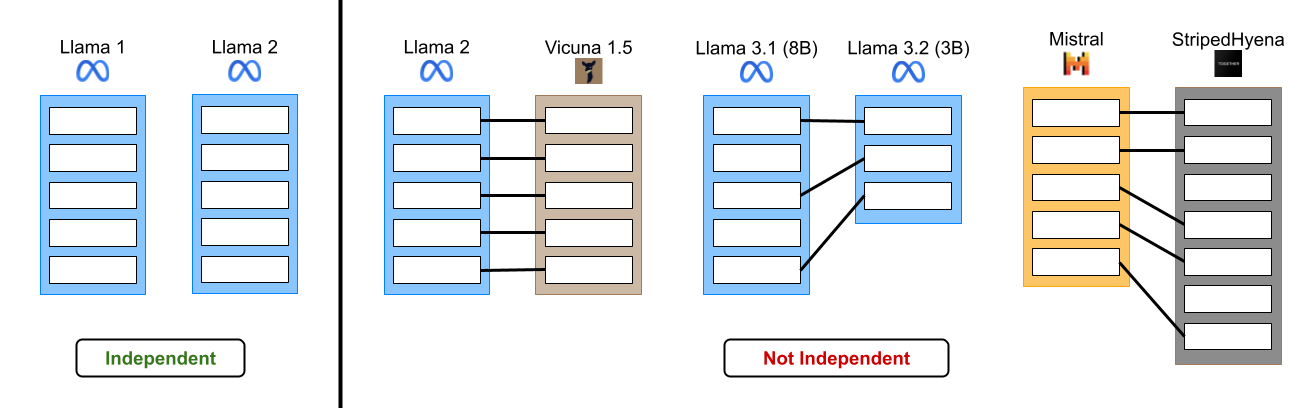}
    \caption{Given the weights of two language models, what relationships can we derive? They could be two models trained from scratch (left). Or, one model could be derived from the other: the dependent model could be a fine-tune, a pruned model, or a partially pruned model (right). We present tests to identify such relationships. 
    }
    \label{fig:figure1}
\end{figure}

A solution to this independence testing problem would enable independent auditors to track provenance of open-weight models. This is pertinent because while open-weight models enable broader access and customization, they also pose potential risks for misuse as they cannot be easily monitored or moderated \citep{kapoor2024openmodels}.
Model developers would also gain an enhanced ability to protect their intellectual property (IP) \citep{2024miquscandal, Peng2023dnnip} and enforce custom model licenses \citep{dubey2024llama3herdmodels,deepseekai2024deepseekv3technicalreport}. 

We consider two settings of the independence testing problem (Table \ref{tab:two_settings}).
In the $\constr$ setting, we make assumptions on training and initialization (essentially, that the training algorithm is equivariant to permuting the hidden units of the random initialization)
that enable us to obtain provably valid p-values. 
The main idea is that under these assumptions we can cheaply simulate many exchangeable copies of each model's weights and compare the value of some test statistic (e.g., cosine similarity of model weights) on each of these copies with the original model pair.
The assumptions generally hold in practice but preclude robustness to adversarial evasion attacks and architectural changes.

For the $\constr$ setting, we evaluate various test statistics on 21 models of the Llama 2 architecture \citep{touvron2023llama2openfoundation}, including 12 fine-tunes of Llama 2 and nine independently trained models, obtaining extremely small p-values for all 69 non-independent model pairs. Notably, our tests retain significant (small) p-values over different fine-tuning methods (e.g., different optimizers) and on models fine-tuned for many tokens from the base model such as Llemma \citep{azerbayev2024llemmaopenlanguagemodel}, which was fine-tuned on an additional 750B tokens from Llama 2 (37.5\% of the Llama 2 training budget). We are also able to confirm that the leaked Miqu-70B model from Mistral is derived from Llama 2-70B.

Next, we consider the unconstrained setting. While the constrained setting is useful for studying the existing ecosystem of open-weight models, simple modifications to model weights and architecture such as permuting hidden units can violate the assumptions of the $\constr$ setting if an adversary applies them after fine-tuning a model.
We address this limitation in the $\unconstr$ setting, wherein we do not make any assumptions on training.
Though we are not able to obtain provably exact p-values in the $\unconstr$ setting, we derive a test whose output empirically behaves like a p-value and reliably distinguishes non-independent models from independent models.
In particular, we first align the hidden units of two models---which may each have different activation types and hidden dimensions---and then compute some measure of similarity between the aligned models. Because of the alignment step, the test is robust to changes in model architecture and various adversarial evasion attacks (including those that break prior work). Moreover, it can localize the dependence: we can identify specific components or weights that are not independent between two models, even when they have different architectures.

\begin{table}[]
    \centering
    \begin{tabular}{|p{7cm}|p{7cm}|}
    \hline 
        \textbf{$\constr$ setting} & \textbf{$\unconstr$ setting} \\ \hline 
        gives exact p-values & does not give exact p-values \\ \hline 
        not robust to permutation & robust to permutations and other adversarial transformations \\ \hline 
        only applies to models of fixed shared architecture & works for models of different architectures and gives localized testing of shared weights \\ \hline 
    \end{tabular}
    \caption{Features of the $\constr$ versus $\unconstr$ problem settings.
    }
    \label{tab:two_settings}
\end{table}

In the unconstrained setting, we evaluate our test on 141 independent model pairs and find that its output empirically behaves like a p-value in the sense that it is close to uniformly distributed in $(0,1]$ over these pairs.
In contrast, it is almost zero for all dependent pairs we test (including those for which we simulate a somewhat strong adversary by retraining entire layers from scratch).
We also employ our test to identify \textit{pruned} model pairs, which occur when a model is compressed through dimension reduction techniques, such as reducing the number of layers or decreasing the hidden dimension by retaining select activations and weights; for example, we identified the precise layers of Llama 3.1-8B from which each of the layers of Llama 3.2-3B and Llama 3.2-1B derive.

We defer a full discussion of related work to Section~\ref{sec:related-work}. The work most closely related to ours is due to \citet{zeng2024humanreadablefingerprintlargelanguage}, who develop various tests to determine whether one model
is independent of another by computing the cosine similarity of the products of certain weight matrices in both models.
They show that their tests are robust to simple adversarial transformations of model weights that preserve model output;
however, we detail in Appendix~\ref{app:breakhuref} other transformations to perturb dependent models that evade detection by their tests, but not by our unconstrained setting tests. 
Additionally, unlike \citet{zeng2024humanreadablefingerprintlargelanguage}, in the $\constr$ setting we obtain exact p-values from our tests. 

\section{Methods}

\subsection{Problem formulation}\label{sec:basics}

Let $f : \Theta \times \mc{X} \to \mc{Y}$ denote a \textit{model} mapping parameters $\theta \in \Theta$ and an input $X \in \mc{X}$ to an output $f(X ; \theta) \in \mc{Y}$. 
We represent a model training or fine-tuning process as a \textit{learning algorithm} $A: \Theta \to \Theta$ that takes as input a set of initial parameters corresponding to either a random initialization or, in the case of fine-tuning, base model parameters. 
Specifically, $A$ includes the choice of training data, ordering of minibatches, and all other design decisions and even the randomness used during training---everything other than the initial model weights. 

Given two models $\firstparam,\secondparam \sim P$ for some joint distribution $P \in \mc{P}(\Theta_1 \times \Theta_2)$,
our goal is to test the null hypothesis
\begin{equation}
    H_0: \firstparam \perp \secondparam,
\end{equation}
where $\perp$ denotes independence of two random variables.
One example of a case where $\firstparam$ and $\secondparam$ might not be independent is if $\secondparam$ is fine-tuned from $\firstparam$, i.e. $\Theta_1 = \Theta_2$ (meaning the two models share the same architecture) and
$\secondparam = A(\firstparam)$ for some learning algorithm $A$.
We treat learning algorithms as deterministic functions.
Thus for $\firstparam = A_1(\theta_1^0)$ and $\secondparam = A_2(\theta_2^0)$, then $\theta_1^0 \perp \theta_2^0$, i.e., the two models having independent random initializations, implies our null hypothesis.

Deep learning models are often nested in nature. 
For example, Transformer models contain self-attention layers and multilayer perceptron (MLP) layers as submodels.
We formalize the notion of a model containing another via the following definition of a submodel. We consider a projection operator that capture the subset of the full model's parameters that are relevant to the submodel.\footnote{For example, in the case of a full Transformer model $\theta$ containing an MLP in a Transformer block, 
$\text{proj}(\theta)$ would return only the weights of that MLP $(G,U,D)$ to pass to $g = f_\text{mlp}$ (see Example \ref{example:glu-mlp}).}

\begin{definition}\label{defn:submodel}
    A model $f : \mc{X} \times \Theta \to \mc{Y}$ \textit{contains} a submodel $g : \mc{X}' \times \Theta' \to \mc{Y}'$ if there exists a projection operator $\txt{proj} : \Theta \to \Theta'$ such that for all $\theta \in \Theta$ we have
\begin{align*}
    f(x;\theta) = f_\txt{out}(g(f_\txt{in}(x);\txt{proj}(\theta)))
\end{align*}
for some functions $f_\txt{in} : \mc{X} \to \mc{X}'$
and $f_\txt{out} : \mc{Y}' \to \mc{Y}$ (which may depend on $\theta$).
\end{definition}

Many of our experiments will specifically involve Transformer models containing MLP layers with Gated Linear Unit (GLU) activations, which are widely used among language models.
We will thus specifically define this type of MLP using the following example.

\begin{example}\label{example:glu-mlp}
    (GLU MLP)
    Let $\gateproj,\upproj \in \R^{h \times d}$ and $\downproj \in \R^{d \times h}$. Let $\act : \R \to \R$ be an element-wise activation function. For $x \in \R^d$ and $\theta = (\gateproj,\upproj,\downproj) \in \Thetamlp^h$, let $\fmlp(x; \theta) \defeq D ( \sigma(Gx) \odot (Ux) )$. Also, for $X \in \R^{s \times d}$ let $\fmlp(X ; \theta) \in \R^{s \times d}$ denote the result of broadcasting $\fmlp$ over the rows of $X$.
\end{example}

In addition to the basic independence testing problem above, we also consider the problem of \textit{localized testing}: testing whether various pairs of submodels among two overall models are independent or not. A prototypical example of a localized testing problem is identifying which layers of a larger model (e.g., Llama 3.1-8B) were used to initialize a smaller model (e.g., Llama 3.2-3B) (in this case, we treat the layers as different submodels).

\subsection{Constrained Setting}

\subsubsection{Testing Framework}

Algorithm~\ref{alg:permtest} ($\permtest$) encapsulates our framework for computing p-values against the null hypothesis in the constrained setting, wherein we simulate $T$ exchangeable copies of the first model $\theta_1$ by applying transformations to its weights. 
The exchangeability of these copies holds under some assumptions on the learning algorithm and random initialization that produced the original model.
We capture these assumptions in the following definitions; together, they define the constrained setting.

\begin{algorithm}[h]\label{algorithm:permtest}
    \DontPrintSemicolon
    \caption{Test for computing p-values ($\permtest$)} \label{alg:permtest}
    \KwIn{Model weights $\theta_1, \theta_2$}
    \SetKwInOut{Parameter}{Parameters}
    \Parameter{test statistic $\phi$; discrete transformation class $\Pi$; sample size $T$}
    \KwOut{p-value $\hat{p} \in (0,1]$}
    $\mathtt{n\_ties} \gets 0$\;
    \For{$t \in 1, \dots, T$}{
        $\pi_t \sim \txt{Unif}(\Pi)$\;
        $\phi_t \gets \phi(\pi_t(\theta_1), \theta_2)$\;
        $\mathtt{n\_ties} \gets \mathtt{n\_ties} + \textbf{1} \{ \phi_t = \phi(\theta_1, \theta_2) \}$\footnotemark\;
    }
    $\xi \sim \txt{Unif}\left(\{0,...,\mathtt{n\_ties}\}\right)$ \tcp{break ties randomly}
    $\hat{p} \gets \frac{1}{T+1}(1 + \xi + \sum_{t=1}^T \textbf{1} \{ \phi_t < \phi(\theta_1, \theta_2) \})$\;
    \Return{$1-\hat{p}$}
\end{algorithm}
\footnotetext{We keep track of the number of ties (when the permuted and original statistic have the same value). To ensure uniform distribution of the p-value, we add a uniformly distributed number in $\{0, \dots, s \}$ to the count before yielding the p-value (line 6), essentially breaking times randomly.}

\begin{definition}[$\Pi$-invariance]\label{defn:perm-invar} 
    Let $\Pi \subset \Theta \to \Theta$. A distribution $P \in \mc{P}(\Theta)$
    is $\Pi$-\textit{invariant} if for $\theta \sim P$ and any $\pi \in \Pi$, the parameters $\theta$ and $\pi(\theta)$ are identically distributed.
\end{definition}

\begin{definition}[$\Pi$-equivariance]\label{defn:perm-equiv-det}
    Let $\Pi \subset \Theta \to \Theta$, $\pi \in \Pi$, and $\init \in \Theta$.
    A learning algorithm $A$ is $\Pi$-\textit{equivariant} if and only if $\pi(A(\theta^0)) = A(\pi(\theta^0))$.
\end{definition}

The main idea underlying $\permtest$ is that so long as $\theta_1 = A(\theta_1^0)$ and $\theta_1^0 \sim P$
for some $\Pi$-equivariant learning algorithm $A$ and $\Pi$-invariant distribution $P$, we can simulate $T$
exchangeable (but not independent) copies $\{\pi_t(\firstparam)\}_{t=1}^T$ of $\firstparam$ by sampling $\pi_t \iid \txt{Unif}(\Pi)$. 
This allows us to efficiently compute an exact p-value without actually repeating the training process of $\firstparam$.
In effect, Definitions~\ref{defn:perm-invar} and \ref{defn:perm-equiv-det} imply that $\pi$ commutes with $A$---i.e., $\pi(A(\firstinit)) = A(\pi(\firstinit))$.
Under exchangeability, the p-value output by $\permtest$ will be uniformly distributed over $\{(i+1)/(T+1)\}_{i=0}^T$.

Standard initialization schemes for feedforward networks exhibit symmetry over their hidden units. This symmetry means that permuting hidden units represents one class of transformations under which any such initialization remains invariant.
Moreover, the gradient of the model's output with respect to the hidden units is permutation equivariant; thus, any learning algorithm whose update rule is itself a permutation equivariant function of gradients (e.g., SGD, Adam, etc.) satisfies Definition~\ref{defn:perm-equiv-det} with respect to these transformations.
An
example of a learning algorithm that is not permutation equivariant is one that uses different learning rates for each hidden unit depending on the index of the hidden unit.

\begin{example}[Permuting hidden units]\label{example:permuting-hidden-units} 
    Let $\theta = (\gateproj,\upproj,\downproj) \in \Thetamlp^h$ parameterize a GLU MLP. Recall $\fmlp(x ; \theta) \defeq D ( \sigma(Gx) \odot (Ux) )$ for some element-wise activation function $\sigma : \R \to \R$. 
    Abusing notation, let $\Pi$ be the set of $h \times h$ permutation matrices such that for $\pi \in \Pi$ we define $\pi(\theta) = (\pi \gateproj, \pi \upproj, \downproj \pi^T )$ (permuting the rows of $G, U$ and the columns of $D$). Observe $\fmlp(x ; \theta) = \fmlp(x ; \pi(\theta))$ and $\pi(\nabla_\theta \fmlp(x;\theta)) = \nabla_{\pi(\theta)} f(x;\pi(\theta))$ for all inputs $x$.
\end{example}

The assumptions we make in the constrained setting suffice for $\permtest$ to produce a valid $p$-value, as we show in the following theorem.
Importantly, the result of the theorem holds (under the null hypothesis) without any assumptions on $\secondparam$. Therefore, a model developer of $\theta_1$ testing other models with our methods can have confidence in the validity of our test without trusting the provider of $\secondparam$.
Of course, if $\secondparam$ does not satisfy the equivariance assumption on training (as in the unconstrained setting), then $\permtest$ is unlikely to produce a low p-value even in cases where $\firstparam$ and $\secondparam$ are not independent (e.g. if an adversary finetunes $\secondparam$ from $\firstparam$ but then afterwards randomly permutes its hidden units).

\begin{theorem}\label{thm:main}
    Let $\phi: \Theta \times \Theta \to \R$ be a test statistic and $\Pi \subset \Theta \to \Theta$ be finite. 
    Let $A : \Theta \to \Theta$ be $\Pi$-equivariant and let $P \in \mc{P}(\Theta)$ be $\Pi$-invariant. For $\theta_1^0 \sim P$, let $\theta_1 = A(\theta_1^0)$. Let $\theta_2 \in \Theta$ be independent of $\theta_1$. 
    Then $\est{p} = \permtest(\firstparam,\secondparam)$ is uniformly distributed on $\{\frac{i+1}{T+1}\}_{i=0}^T$.
\end{theorem}
\begin{proof}
    We assume $\Pi$ is finite so that $\txt{Unif}(\Pi)$ is well-defined.
    From our assumptions on $A$ and $P$ and the fact that $\{\pi_t\}_{t=1}^T$ are independently drawn, it follows that the collection $\{\pi_t(\firstparam)\}_{t=1}^T$ comprises $T$ exchangeable copies of $\firstparam$.
    The independence of $\firstparam$ and $\secondparam$ thus implies $\{(\pi_t(\firstparam),\secondparam)\}_{t=1}^T$ comprises $T$ exchangeable copies of $(\firstparam,\secondparam)$. Because we break ties randomly, by symmetry it follows that $\phi(\theta_1, \theta_2)$ will have uniform rank among $\{\phi_t\}_{t=1}^T$.
\end{proof}

One notable (non-contrived) category of deep learning algorithms that are \textit{not} permutation equivariant are those that employ 
dropout masks to hidden units during training. In our framework, the dropout masks are specified in the deterministic learning algorithm $A$. Once we fix a specific setting of mask values in $A$, this algorithm will not be permutation equivariant unless the individual dropout masks are all permutation invariant (which is highly unlikely).
For completeness, we generalize the result of Theorem~\ref{thm:main} to apply to randomized learning algorithms that satisfy a notion of equivariance in distribution (which includes algorithms that use dropout) in Appendix~\ref{sec:randomized-alg}.
However, throughout the main text we will continue to treat learning algorithms as deterministic for the sake of simplicity, and also since dropout typically is no longer used in training language models \citep{chowdhery2022palmscalinglanguagemodeling}.

\subsubsection{Test Statistics}
\label{sec:test-stats}

We have shown $\permtest$ produces a valid p-value regardless of the test statistic $\phi$ we use.
The sole objective then in designing a test statistic is to achieve high statistical power: we would like $\est{p} = \permtest(\firstparam,\secondparam)$ to be small when $\firstparam$ and $\secondparam$ are not independent.
The test statistics we introduce in this section apply to any model pair sharing the same architecture. 

Prior work \citep{xu2024instructionalfingerprintinglargelanguage} proposed testing whether two models are independent or not based on the $\ell_2$ distance between their weights, summed over layers. Specifically
for a model with $L$ layers parameterized by $\Theta = \Theta_1 \times ... \times \Theta_L$, with $\firstparam = (\firstparam^{(\ell)})_{\ell=1}^L$ and $\secondparam = (\secondparam^{(\ell)})_{\ell=1}^L$,
let $\ltwo(\firstparam,\secondparam) \defeq - \sum_{i=1}^L \ell_2(\theta_1^{(i)},\theta_2^{(i)})$.
We can obtain p-values from $\ltwo$ by using it within $\permtest$.
However, a major limitation is that in order to obtain a p-value less than $1/(T+1)$ we must recompute $\ltwo$ at least $T$ times; the effective statistical power of our test using $\ltwo$ is therefore bottlenecked by computation.

To address this limitation, we propose a family of test statistics whose distribution under the null is identical for \textit{any} model pair. Consider 
$m, n \in \N$ and some function $M : \Theta \to \R^{n \times m}$ that maps model weights to a matrix, such as returning a specific layer's weight matrix.
The proposed test statistics all share the following general form based on Algorithm~\ref{alg:speartest} ($\speartest$) for varying $M$:
\begin{align}\label{eqn:phi-m}
    \phi_M(\firstparam,\secondparam) \defeq \spearman(\speartest(M(\firstparam),M(\secondparam)),[1,...,n]),
\end{align}
where $\spearman$ is the Spearman rank correlation (Algorithm~\ref{algorithm:spearman-pvalue}).

\begin{algorithm}[h]\label{algorithm:speartest}
    \DontPrintSemicolon
    \caption{Cosine similarity matching ($\speartest$)}
    \label{alg:speartest}
    \KwIn{Matrices $\firstM, \secondM$ with $h$ rows}
    \SetKwInOut{Parameter}{Parameters}
    \KwOut{Permutation $\pi : [h] \to [h]$}
    \For{$i \in 1, \dots, h$}{
        \For{$j \in 1, \dots, h$}{
            $C_{i,j} \gets \dcos((\firstM)_i,(\secondM)_j)$\;
        }
    }
    $\pi \gets \match(C)$\;
    \Return{$\pi$}
\end{algorithm}

Equation~\eqref{eqn:phi-m} is applicable to any model architecture $\Theta$ for which we can define a suitable matrix valued function $M$ of model parameters. For example, $M$ could directly extract a weight matrix or activation matrix from a model layer (based on some set of inputs), with each row corresponding to a hidden unit. We use $\speartest$ to align the rows of the two extracted matrices and then pass this alignment to $\spearman$ to compute the Spearman rank correlation \citep{spearmanrank} of this alignment with the identity map between rows.
We describe matching in Algorithm~\ref{alg:speartest}, wherein $\dcos$ denotes cosine similarity function and $\match$ denotes the algorithm of \citet{Ramshaw2012OnMA} we use to solve the matching problem. 

\begin{algorithm}[h]\label{alg:spearman-pvalue}
    \DontPrintSemicolon
    \caption{Deriving p-values from Spearman correlation ($\spearman$) from \cite{spearmanrank}}
    \label{algorithm:spearman-pvalue}
    \KwIn{Permutations $\pi_1, \pi_2: [h] \to [h]$}
    \SetKwInOut{Parameter}{Parameters}
    \KwOut{p-value $\hat p \in (0,1]$}
    $r \gets 1 - \frac{6 \sum (\pi_1[i] - \pi_2[i])^2}{h(h^2-1)}$\; 
    $t \gets r \sqrt{\frac{h-2}{1-r^2}}$\;
    $\hat p \gets \mathbb{P}(T_{n-2} > t)$ \; 
    \Return{$\hat p$}
\end{algorithm}

The idea of the test is that for two dependent models, each row of $M(\firstparam)$ should be similar to its counterpart in $M(\secondparam)$; thus, the alignment found by $\spearman$ will be close to the identity map.
Meanwhile, so long as $M$ is a $\Pi$-equivariant map (Definition~\ref{defn:equivariant-map}), then $\phi_M(\theta_1,\theta_2)$ under the null yields valid p-values (see Theorem 2), 
so we can use the more computationally-efficient Algorithm~\ref{algorithm:spearman-pvalue} to convert statistics to p-values instead of running $\permtest$.

\begin{definition}\label{defn:equivariant-map}
    (equivariant map) A matrix-valued function $M : \Theta \to \R^{n \times m}$ is $\Pi$-equivariant with respect to a class of transformations $\Pi : \Theta \to \Theta$ if there exists a bijection between $\Pi$ and the set of $n \times n$ permutation matrices such that $M(\pi(\theta)) = \pi M(\theta)$ for all $\theta \in \Theta$ and $\pi \in \Pi$.
\end{definition}

\begin{theorem}
Let $M: \Theta \to \R^{n \times m}$ be a $\Pi$ equivariant map and let $P \in \mc{P}(\Theta)$ be $\Pi$-invariant. Let $\firstparam,\secondparam \in \Theta$ be independent random variables, with $\firstparam = A(\firstinit)$ for $\firstinit \sim \firstP$. Then $\phi_M(\theta_1,\theta_2)$ is uniformly distributed on $(0,1]$.
\end{theorem}
\begin{proof}
    As $M$ is a $\Pi$-equivariant map, if $\firstparam \perp \secondparam$ then letting $\pi = \match(C)$ in $\speartest$ is equivalent in distribution to sampling $\pi \sim \txt{Unif}(\Pi)$. Then the output of $\speartest$ is identical in distribution for \textit{any} pair of independent models, and can be converted to a p-value using $\spearman$ and the distribution for the Spearman correlation coefficient (t-distribution with $h-2$ degrees of freedom).
\end{proof}

We can choose various different functions for $M$, with each yielding a valid test statistic.
We focus our experiments on Transformer models consisting of a series of $L$ Transformer blocks. Each block contains a GLU MLP submodel, where $M(\theta)$ represents either the  up projection weights or the hidden-layer activations of these submodels.
In particular, let $U^{(\ell)}(\theta) \in \R^{h \times d}$ denote the first layer up projection weights of the MLP in the $\ell$-th block, where $h$ is the hidden dimension and $d$ is the input dimension, and let $H^{(\ell)}(\theta) \in \R^{h \times (N \cdot s)}$ denote the (flattened) hidden activations obtained from passing $N$ length $s$ input sequences $X \in \R^{N \times s \times d}$ to the same MLP module (the test is valid for any $X$; we will specify later how we choose $X$ in our experiments). 
The two main test statistics we will employ in our experiments are $\csw$ and $\csh$.

Both $U^{(\ell)}$ and $H^{(\ell)}$ are equivariant with respect to permuting the hidden units of the corresponding MLP, so we can directly interpret the outputs of $\csw$ and $\csh$ as p-values.
{Moreover, we can separately permute the hidden units of the MLP in the $\ell$-th block without changing the inputs or outputs of the other blocks.
Thus, as we show in Theorem~\ref{thm:fisher},
we can aggregate the p-values from $\csw$ and $\csh$ across blocks using Fisher's method (\cite{fisher_qa}) to obtain a more powerful test in Algorithm~\ref{alg:fisher} ($\fisher$). Specifically, Fisher's method of aggregating p-values requires independent tests from the same null hypothesis.

\begin{algorithm}[h]\label{algorithm:fisher}
    \DontPrintSemicolon
    \caption{Aggregating p-values ($\fisher$)}
    \label{alg:fisher}
    \KwIn{p-values $\{\est{p}^{(i)}\}_{i=1}^L$}
    \KwOut{p-value $\hat{p} \in (0,1]$}
    $\xi \gets \sum_{i=1}^L \log \est{p}^{(i)}$\;
    $\hat{p} \gets 1-\mathbb{P}(\chi^2_{2L} < -2\xi)$\;
    \Return{$\hat p$}
\end{algorithm}

\begin{theorem}
    \label{thm:fisher}
    Consider block indices $i,j \in [L]$ with $i \neq j$ for models with $L$ blocks. Suppose for $\ell \in \{i,j\}$ that
    \begin{itemize}
        \item[1.] $M^{(\ell)} : \Theta \to \R^{h \times N}$ is equivariant with respect to $\Pi^{(\ell)}$, i.e.,
              for any $\theta \in \Theta$ and $\pi^{(\ell)} \in \Pi^{(\ell)}$ we have
              \begin{align*}
                  M(\pi^{(\ell)}(\theta)) = \pi^{(\ell)} M(\theta).
              \end{align*} 
        \item[2.] $A$ is a $\Pi^{(\ell)}$-equivariant learning algorithm and $P \in \mc{P}(\Theta)$ is a $\Pi^{(\ell)}$-invariant distribution.
    \end{itemize}
    Let $\firstparam,\secondparam \in \Theta$. If $\firstparam \perp \secondparam$ for $\firstparam = A(\firstinit)$ with $\firstinit \sim P$,
    then
    \begin{align*}
        \speartest(M^{(i)}(\firstparam),M^{(i)}(\secondparam)) \perp \speartest(M^{(j)}(\firstparam),M^{(j)}(\secondparam)).
    \end{align*}
\end{theorem}
\begin{proof}
    Let $\firstparam' \sim A(\pi_1^{(i)} \circ \pi_2^{(j)} (\firstinit))$ for $\pi_1,\pi_2 \iid \txt{Unif}(\Pi)$. Then
    $\firstparam'$ is an independent copy of $\firstparam$ since taking the composition $\pi_1^{(i)} \circ \pi_2^{(j)} (\firstparam)$
    yields an independent copy of $\firstparam$ for any $\pi_1,\pi_2 \in \Pi$.
    From $\firstparam \perp \secondparam$, it follows for $\ell \in \{i,j\}$
    that $\speartest(M^{(\ell)}(\firstparam'),M^{(\ell)}(\secondparam))$ is identically distributed to
    $\speartest(M^{(\ell)}(\firstparam),M^{(\ell)}(\secondparam))$.
    The result then follows from the fact $\speartest$ is equivariant with respect to permuting the rows of its arguments: in particular,
    for any $\pi \in \Pi$ we have $\speartest(\pi \firstM,\secondM) = \pi \speartest(\firstM,\secondM)$.
\end{proof}

Recall $\csw$ and $\csh$ are functions of $\speartest(M^{(\ell)}(\firstparam),M^{(\ell)}(\secondparam))$ respectively for
$M^{(\ell)} = U^{(\ell)}$ and $M^{(\ell)} = H^{(\ell)}$, both of which satisfy the assumptions of the theorem.
Thus, the result of the theorem applies to both these test statistics, and the independence of the p-values from these test statistics across blocks follows directly from the independence of the statistics themselves. Hence, we can use Fisher's method and Algorithm \ref{alg:fisher} to aggregate p-values from $\csw$ or $\csh$ values to obtain a more powerful test. 

\subsection{Unconstrained Setting}
\label{robustness}

For the $\unconstr$ setting, our goal is to design a test that applies to models of different architectures and is robust to output-preserving transformations of model weights.
Recall our tests for the constrained setting satisfy neither of these desiderata: these tests assume both models have the same number of hidden units, and it is easy to fool them without changing the output of a model by permuting the order of the hidden units in the model.

Our robust test reposes on the design of $\phi_M$ in equation~\eqref{eqn:phi-m}. The goal is to identify two matrix valued functions of model parameters $M,M' : \Theta \to \R^{n \times m}$ that jointly satisfy the following condition: any output-preserving transformation of model parameters must transform both $M$ and $M'$ in the same way. Then, whereas previously we would correlate $\speartest(M(\theta_1),M(\theta_2))$ with the identity permutation, we instead define
\begin{align}\label{eqn:robust-test}
    \phi_{M,M'} \defeq \spearman(\speartest(M(\theta_1),M(\theta_2)),\speartest(M'(\theta_1),M'(\theta_2)).
\end{align}

The above goal is aspirational in the sense that for any nontrivial deep learning model we are not able to fully enumerate the set of transformations of model parameters to which model output is invariant; 
nonetheless, it will serve as a useful guiding principle for designing our robust test under the framework of equation~\eqref{eqn:robust-test}.
We organize the description of our full robust test---which is generally applicable to a variety of model architectures---into two parts: first, in Section \ref{sec:gluarchitecture} we instantiate equation~\eqref{eqn:robust-test} to obtain a test for GLU MLP models. Then, in Section \ref{sec:generalized} we use our GLU MLP test as a primitive for designing a test that applies to general deep learning models (including those which do not contain any GLU MLP submodels).

\subsubsection{Testing GLU models}
\label{sec:gluarchitecture}

Recalling our definition of a GLU MLP model in Example~\ref{example:glu-mlp}, for $k \in \{1,2\}$ let $\theta_k = (\gateproj_k,\upproj_k,\downproj_k) \in \Thetamlp^{h_k}$, and with inputs $X \in \R^{d \times N}$ let $H_\txt{up}(\theta_k) = U_k X \in \R^{\max\{h_1,h_2\} \times N}$ be the output of the up projection operation and let $H_\txt{gate}(\theta_k) = G_k X \in \R^{\max\{h_1,h_2\} \times N}$ be the output of the gate projection operation (with appropriate zero-padding when $h_1 \neq h_2$).
Due to the element-wise product operation, we conjecture that in general it is not possible to permute the rows of $\gateproj_k$ while preserving the output of $\theta_k$ without permuting the rows $\upproj_k$ in the same way, and so we use 
$\phi_{M,M'}$ with $M = H_\txt{gate}$ and $M' = H_\txt{up}$ for our GLU MLP test. Henceforth, we will shorthand this test as $\rob$.

As with the constrained setting, we focus much of our experiments on Transformer models, which recall consist of a series of $L$ Transformer blocks that each contain a GLU MLP submodel. Adopting the notational conventions of Section~\ref{sec:test-stats}, we can apply our GLU MLP test to the $\ell$-th block by taking $M = H_\txt{gate}^{(\ell)}$ and $M' = H_\txt{up}^{(\ell)}$, where like before (in the case of $\csh$) we obtain the activation inputs for each block by computing a forward pass through the full model over a set of length $s$ sequences of input tokens.

We can aggregate the results of these tests over blocks using $\fisher$, like we do for $\csw$ and $\csh$ in the constrained setting.
Alternatively, to perform localized testing we can apply the test to all possible $O(L^2)$ pairs of blocks between two Transformer models if we suspect that certain blocks from one model served as the initializations for different blocks in the other model. Specifically, we can test the $i$-th block of $\firstparam$ and the $j$-th block of $\secondparam$ using
\begin{align}
\label{eqn:matching-blocks}
    \rob^{(i,j)} \defeq \spearman(\speartest(H_\text{gate}^{(i)}(\theta_1),H_\text{gate}^{(j)}(\theta_2)),\speartest(H_\text{up}^{(i)}(\theta_1),H_\text{up}^{(j)}(\theta_2))).
\end{align}

\subsubsection{Beyond GLU Models}
\label{sec:generalized}

Thus far we have focused on models $f : \mc{X} \times \Theta \to \mc{Y}$ containing a GLU MLP submodel.
In particular, recalling Definition~\ref{defn:submodel},
we have assumed for some $\txt{proj}_\txt{mlp} : \Theta \to \Thetamlp^h$ that
\begin{align}\label{eqn:og-model}
    f(x;\theta) = f_\txt{out}(\fmlp(f_\txt{in}(x);\txt{proj}_\txt{mlp}(\theta))).
\end{align}
Now, our goal will be to test more general types of models.
In particular, we generalize to an arbitrary alternative submodel $f_\txt{alt} : \R^d \times \Theta_\txt{alt} \to \R^d$ with $\txt{proj}_\txt{alt} : \Theta \to \Theta_\txt{alt}$ such that
\begin{align}\label{eqn:alt-model}
    f(x;\theta) = f_\txt{out}(f_\txt{alt}(f_\txt{in}(x);\txt{proj}_\txt{alt}(\theta))).
\end{align}

In order to test whether two models $\firstparam,\secondparam \in \Theta$ of the more general form in equation~\eqref{eqn:alt-model} are independent, we will first construct proxy models of the form in equation~\eqref{eqn:og-model} and then apply our previous test $\rob$ to these proxy models.
We construct these proxy models by leveraging
the fact that $f_\txt{alt}$ shares the same input and output space with $\fmlp$.
Specifically, for $k \in \{1,2\}$ we first learn parameters $\est{\theta}_k \in \Thetamlp^h$ so that $\fmlp(\cdot \ ;\est{\theta}_k)$ approximates $f_\txt{alt}(\cdot \ ;\txt{proj}_\txt{alt}(\theta_k))$. We then return $\rob(\est{\theta}_1,\est{\theta}_2)$.
We capture this two-stage process in Algorithm~\ref{algorithm:general-test}.

\begin{algorithm}[H]\label{algorithm:general-test}
    \DontPrintSemicolon
    \caption{Generalized robust test}
    \KwIn{Model parameters $\firstparam, \secondparam \in \Theta$}
    \SetKwInOut{Parameter}{Parameters}
    \Parameter{distribution $P$ over $\R^d$}
    \KwOut{$\est{p} \in [0,1]$}
    \For{$k \in \{1,2\}$}{
        $\est{\theta}_k \gets \arg\min_{\est{\theta}} \Ep_{x \sim P} \left[ \norm{f_\txt{alt}(x;\txt{proj}_\txt{alt}(\theta_k)) - f_\txt{mlp}(x;\est{\theta}_k)}^2\right]$
    }
    \Return{$\est{p} \gets \rob(\est{\theta}_1,\est{\theta}_2)$}
\end{algorithm}

Perhaps surprisingly, we show that Algorithm~\ref{algorithm:general-test} is effective in practice at distinguishing independent versus non-independent models. The hidden dimension $h$ and input distribution $P$ with which we learn the GLU MLP are hyperparameters of the test. See Section \ref{sec:adversarial-experiments} for details. 
\section{Experimental Results}

\subsection{Constrained setting}
\label{sec:constrained-results}

We first validate the effectiveness of our tests in the constrained setting on open-weight language models --- 21 models trained with the Llama-7B architecture with public documentation on ground truth model independence. These models all contain $L = 32$ GLU MLPs, each part of its own Transformer block.

We run experiments with three different tests. Each test comprises two elements: a test statistic along with a method for computing p-values from the statistic.
For the first test, we use $\ltwo$ and compute p-values via $\permtest$ with $T = 99$.
The equivariant transformation class $\Pi$ is the set of permutations over both the hidden units of each MLP (see Example~\ref{example:permuting-hidden-units}) and the embedding dimension of the model (i.e., the inputs passed to the both the MLP and self-attention layers in each block); we defer the precise definition of $\Pi$ in this case to Appendix \ref{app:llama_permutation}.
For the other two tests, we compute p-values by directly aggregating the outputs of (respectively) $\csw$ and $\csh$ over $\ell \in [L]$ using $\fisher$.
We sample sequences of 4096 tokens uniformly at random from the models' vocabulary, then compute a forward pass through the full model while storing the MLP hidden layer activations, and use the input activations to the GLU MLP in the $\ell$-th layer as the activations to compute $\csh$.

In addition to these three tests, we report the Jensen-Shannon divergence between next token output distributions $\jsd$ \citep{jsd}.
Since $\jsd$ is (by definition) invariant to any transformation of weights that does not affect model output, we cannot compute meaningful p-values using $\permtest$; instead, in our experiments we report the raw test statistic value as a baseline reference.

\subsubsection{Results for Llama model tree}

\begin{figure}
    \centering
    \includegraphics[width=0.95\textwidth]{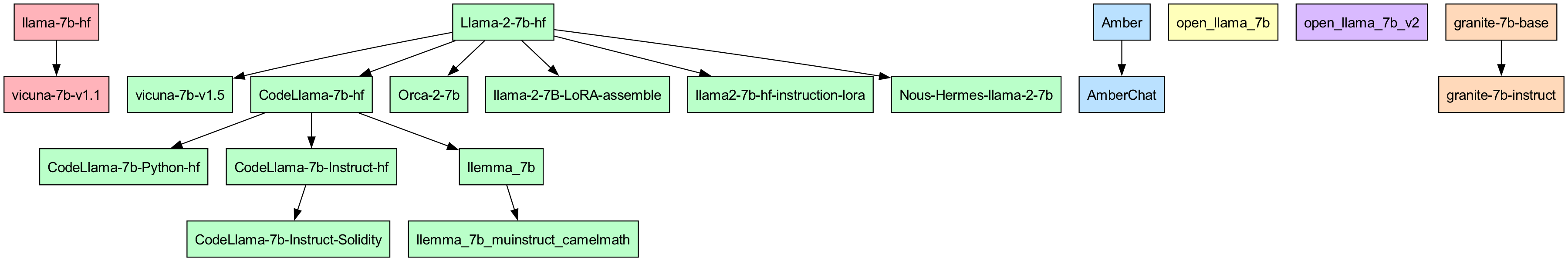}
    \caption{We enumerate the public Llama-7B models and delineate the sets of dependent model pairs by color.}
    \label{fig:model-lineage-forest}
\end{figure}

The 21 models we evaluated, shown in Figure \ref{fig:model-lineage-forest}, include 6 base (trained from scratch) models of Llama-7B architecture and a variety of their finetunes, which form six disjoint sets of models stemming from a diverse mix of industry labs and non-profits \citep{azerbayev2024llemmaopenlanguagemodel, sudalairaj2024lablargescalealignmentchatbots, liu2023llm360, li2023camelcommunicativeagentsmind}.
We consider any pair of models in the same tree as dependent and all other pairs as independent.
We include examples of further fine-tunes (e.g., \texttt{llemma\_7b}) of fine-tunes (e.g., \texttt{CodeLlama-7b-hf}) among the models we test.
We will mostly refer to models using by their Hugging Face identifiers, without the organization names for brevity.
We report results for a subset of these pairs involving base model \texttt{Llama-2-7b-hf} in Table~\ref{tab:basic_statistics_llama2} while deferring the rest and the full experimental setup details to Appendix \ref{app:exp_results}.

\begin{table}[h]
    \centering
    \begin{tabular}{c c|c|cc c c}
    \hline 
         & &  & & & p-values & \\ 
        $\theta_1$ & $\theta_2 $ & Indep.? & $\jsd$ (log) & $\ltwo$ & $\csw$ & $\csh$ \\ \hline 
         \texttt{Llama-2-7b-hf} & \texttt{llama-7b-hf} & \ding{51} & -11.10 & 0.98 & 0.60 & 0.25 \\
         \texttt{Llama-2-7b-hf} & \texttt{vicuna-7b-v1.1} & \ding{51} & -10.40 & 0.63 & 0.16& 0.64 \\
         \texttt{Llama-2-7b-hf} & \texttt{Amber} & \ding{51} & -10.69 & 0.75 & 0.36 & 0.88 \\
         \texttt{Llama-2-7b-hf} & \texttt{open-llama-7b} & \ding{51} & -8.38 & 0.26 & 0.36 & 0.71 \\
         \hline 
         \texttt{Llama-2-7b-hf} & \texttt{vicuna-7b-v1.5} & \ding{55} & -10.87 & 0.01 & $\varepsilon$ & $\varepsilon$ \\ 
         \texttt{Llama-2-7b-hf} & \texttt{CodeLlama-7b-hf} & \ding{55} & -10.62 & 0.01 & $\varepsilon$ & $\varepsilon$\\
         \texttt{Llama-2-7b-hf} & \texttt{llemma-7b} & \ding{55} & -10.24 & 0.01 & $\varepsilon$ & $\varepsilon$ \\
         \texttt{Llama-2-7b-hf} & \texttt{Orca-2-7b} & \ding{55} & -10.34 & 0.01 & $\varepsilon$ & $\varepsilon$ \\ \hline 
    \end{tabular}
    \caption{We report various constrained setting test statistics with $\firstparam$ as \texttt{Llama-2-7b-hf} and $\secondparam$ ranging over the listed models. The ``independent'' column is the ground truth. Here, $\varepsilon =$ 2.2e-308 (numerical underflow for a 64-bit float). We find our proposed tests $\csw$ and $\csh$ distinguish independent versus non-independent model pairs with high statistical power.
    }
    \label{tab:basic_statistics_llama2}
\end{table}

Consistent with prior work \cite{xu2024instructionalfingerprintinglargelanguage}, 
we find that $\jsd$ does not reliably distinguish independent versus dependent model pairs. For example, \texttt{CodeLlama-7b-hf} exhibits a larger divergence with \texttt{Llama-2-7b-hf} than the independently-trained models \texttt{llama-7b-hf} and \texttt{Amber}.

All other test statistics reliably distinguish independent versus dependent pairs; in particular, the p-values we obtain using the other test statistics are negligible for all dependent pairs (for $\ltwo$, because we run $\permtest$ with $T = 99$ for computational reasons, we cannot obtain a p-value less than $0.01$).
Notably, in contrast to our findings, prior work \citep{xu2024instructionalfingerprintinglargelanguage} argued that the $\ell_2$ distance between model parameters is not a reliable indicator of independence, in the sense that the $\ell_2$ distance between dependent pairs is sometimes larger than that of independent pairs (similar to the case of $\jsd$); the key difference is that \citet{xu2024instructionalfingerprintinglargelanguage} report the raw $\ell_2$ distance whereas we obtain p-values from the raw distances using $\permtest$.
We hypothesize that $\permtest$ effectively standardizes the raw distances.

Finally, we evaluated our tests on models of different architectures besides the Llama-7B architecture. We ran $\csw$ on four 70B parameter models, each with the same Llama 2-70B architecture, with results shown in Table \ref{tab:70b}. Notably, we verify that Miqu-70B from MistralAI is not independent from Llama 2-70B \citep{2024miquscandal}. 

\begin{table}[]
    \centering
    \begin{tabular}{c|c|c}
    \hline 
        $\theta_1$ & $\theta_2$ & $\csw$ \\ \hline
        \texttt{Llama-2-70b-hf} & \texttt{miqu-1-70b-pytorch} & $\varepsilon$ \\ 
        \texttt{Llama-2-70b-hf} & \texttt{Llama-3.1-70B} & 0.571 \\ 
        \texttt{Llama-2-70b-hf} & \texttt{Palmyra-Fin-70B-32K} & 0.539 \\ \hline 
    \end{tabular}
    \caption{We evaluate the constrained setting test involving up-projection weights $\csw$ (aggregated with $\fisher$) with $\theta_1$ as \texttt{Llama-2-70b-hf} and $\theta_2$ ranging over the listed models. Here, $\varepsilon = $ 2.2e-308, notably suggesting that \texttt{Llama-2-70b-hf} and the leaked Mistral model \texttt{miqu-1-70b-pytorch} are not independent.}
    \label{tab:70b}
\end{table}

\subsection{Unconstrained setting}
\label{sec:adversarial-experiments}

Next, we evaluate our robust test $\rob$ in the unconstrained setting, which encompasses varying model architecture and adversarial evasion attacks. We also examine the internal values within $\rob$ from the $\speartest$ algorithm to perform localized testing.

We first assess the previous 21 models of the Llama-7B architecture. We compute $\rob$ with the gate and up-projection matrices $M = H_\text{gate}^{\ell}$ and $M' = H_\text{up}^{\ell}$ of each MLP in block $\ell \in [L]$, and aggregate them with $\fisher$. We obtain the activations in the MLPs by using input sequences sampled from WikiText-103 and computing a forward pass through the full model, 
with results on all model pairs in Appendix \ref{app:robust_results}. 

We find that the distribution of $\rob$ on independent model pairs is close to uniform 
(Figure \ref{fig:phimatchlines}), whereas across all non-independent model pairs the statistic is at most $\eps$. 
Unlike the $\constr$ setting, where the p-values are valid by construction, the output of the robust test does not enjoy such theoretical guarantees; however, Figure~\ref{fig:phimatchlines} suggests that even in the $\unconstr$ setting our statistic $\rob$ behaves like a p-value, i.e. that it is uniformly distributed on $[0,1)$ under the null hypothesis. 

\begin{figure}[h]
\centering
  \begin{subfigure}[]{0.45\textwidth}
    \includegraphics[width=\textwidth]{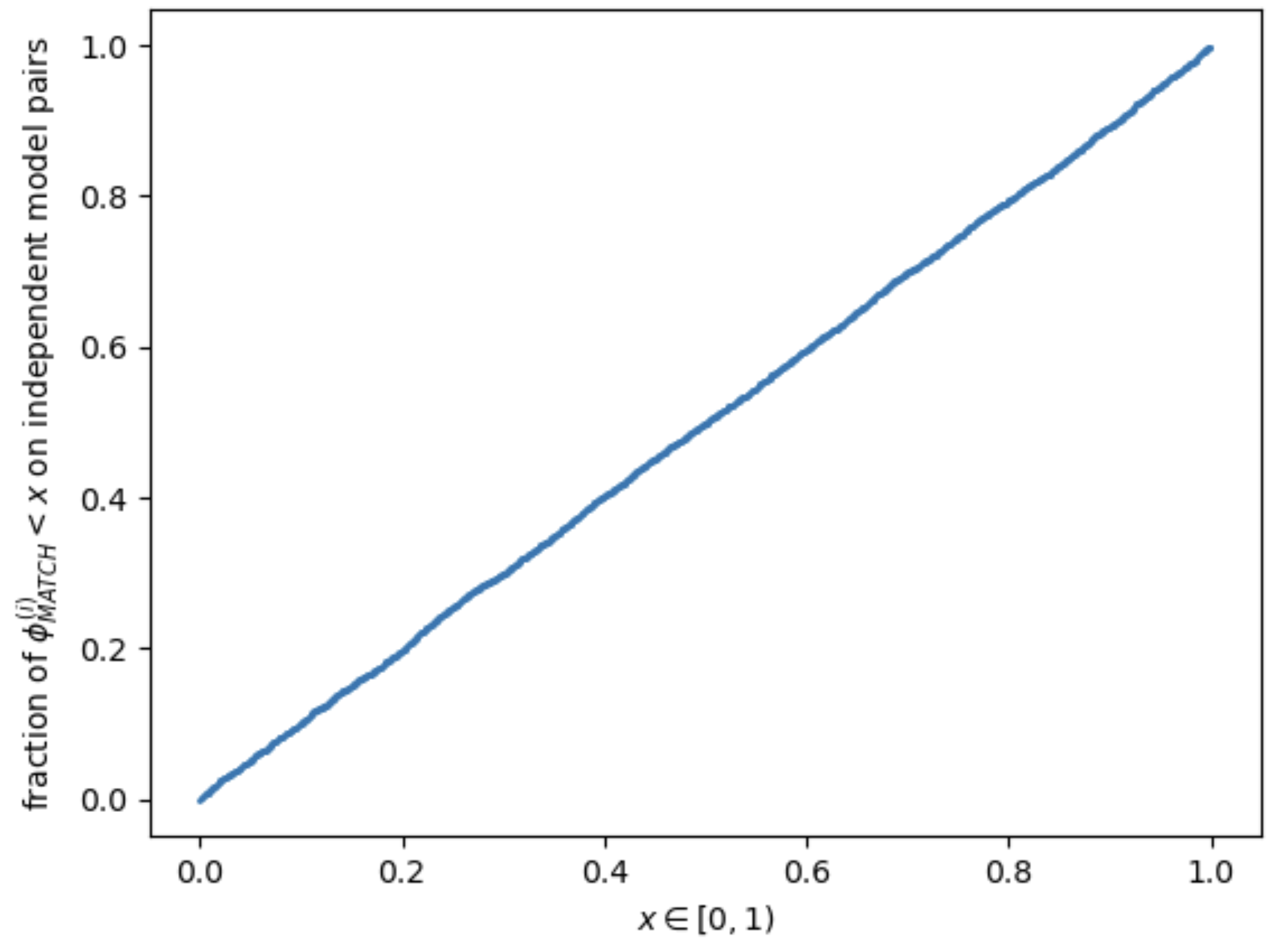}
    \caption{Plot of $x \in [0,1)$ vs. the fraction of $\rob^{(i)}$ (across all MLP blocks) of independent model pairs less than $x$.}
    \label{fig:robust_stat_all}
  \end{subfigure}
  \begin{subfigure}[]{0.45\textwidth}
    \includegraphics[width=\textwidth]{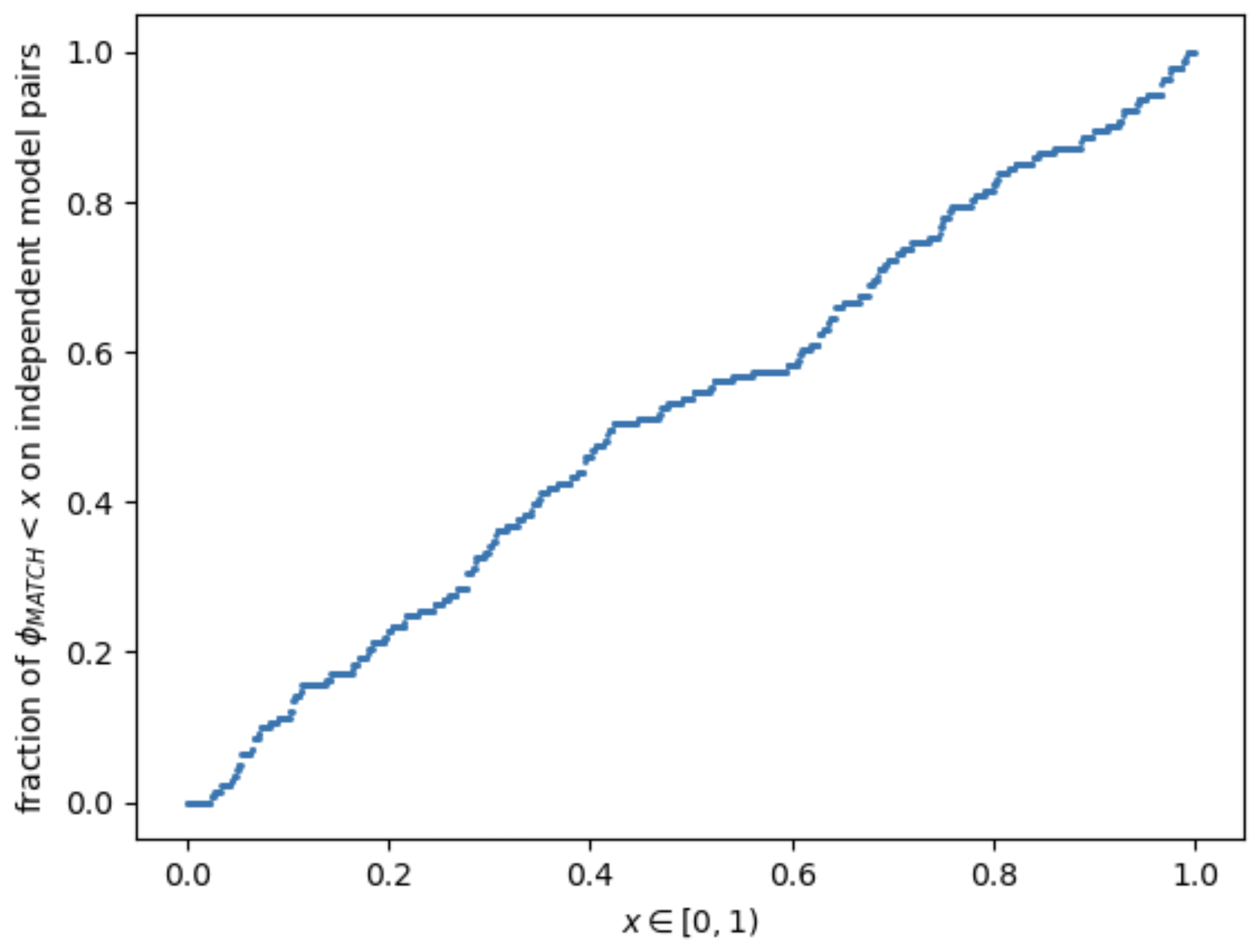}
    \caption{Plot of $x \in [0,1)$ vs. the fraction of $\rob$ ($\rob^{(i)}$ aggregated with $\fisher$ across all MLP blocks) of independent model pairs less than $x$.}
    \label{fig:robust_stat_agg}
  \end{subfigure}
  \caption{We plot the fraction of the statistic $\rob$ less than $x \in [0,1)$, aggregated with $\fisher$ and not for independent model pairs. Both plots roughly follow the line $y=x$, i.e. a uniform distribution in $[0,1)$ under the null, meaning $\rob$ empirically acts as a p-value.}
  \label{fig:phimatchlines}
\end{figure}

We next validate our tests on other model pairs with differing architectures: we compare the weights of the hybrid model \texttt{StripedHyena-Nous-7B} \citep{stripedhyena} with \texttt{Mistral-7B-v0.1} (where some layers of \texttt{StripedHyena-Nous-7B} are taken from \texttt{Mistral-7B-v0.1} and others are not) and find non-independent parameters via $\csw$. We compute $\csw$ on all parameters, which allows us to identify non-independence between specific parameters of the models, such as the self-attention matrices, rather than as models as a whole. We report values of $\csw$ among parameters of the embedding layer and first Transformer block in Table \ref{tab:stripedhyena} in Appendix \ref{app:stripedhyena}. From the small p-values we infer that some parameters, including the embedding layer and some self-attention matrices, were likely shared between the two models. 

\subsubsection{Testing identically distributed but independent models}
\label{sec:independent-similar-models}

We further evaluate the validity of our robust test through ablations by training two near-identical models that only differ on select sources of randomness. Specifically, we would like to check that $\rob$ does not incorrectly detect two similar (trained using the same learning algorithm) but independent (randomly initialized) models, as non-independent. 

To do this, we randomly initialize a model with the OLMo (7B) architecture \citep{groeneveld2024olmoacceleratingsciencelanguage} and train it on the \texttt{Dolma v1\_7} dataset \citep{soldaini2024dolmaopencorpustrillion}. We train a second model with independently chosen initialization and data ordering, but on the same dataset.
By only changing initialization and data ordering, we have two very similar models (trained with essentially the same learning algorithm $A$), yet are independent due to their random initializations. 

We keep checkpoints for both seeds after 100M, 1B, 10B, and 18B train tokens and evaluate the statistics $\csw, \csh$, and $\rob$ on the two models at each training checkpoint, reported in Table \ref{tab:olmo_indp}. We highlight that the p-values are broadly distributed, indicating that $\rob$ is valid even on two similarly-trained but independent models (the other tests are valid by construction but we report their results anyways just for reference).
We find that all test statistics work well, and there is also little difference in the results at different training checkpoints.

\begin{table}[]
    \centering
    \begin{tabular}{c|c|c|c|c|c}
    \hline 
        \# train tokens & $\csw$ & $\csh$ & $\ltwo$ & $\rob$ & $\jsd$ (log) \\ \hline 
        100M & 0.641 & 0.119 & 0.07 & 0.809 & -11.81 \\ 
        1B & 0.789 & 0.483 & 0.06 & 0.443 & -11.05 \\ 
        10B & 0.707 & 0.277 & 0.93 & 0.343 & -11.28 \\ 
        18B & 0.819 & 0.141 & 0.64 & 0.027 & -11.03 \\ \hline 
    \end{tabular}
    \caption{We evaluate the statistics $\csw, \csh$, $\ltwo$, $\rob$, and $\jsd$ on four training checkpoints between two similar but independently-trained OLMo models (weights initialized independently). We find that our proposed statistics are uniformly distributed, meaning no false positives even for two similarly-trained models.}
    \label{tab:olmo_indp}
\end{table}

\subsubsection{Simulating strong-ish adversaries}
\label{sec:mlp-retrain}

A significant difficulty in evaluating the robustness of our test $\rob$ to adversarial transformations---in particular, transformations of weights that preserve model output---is that we cannot exhaustively enumerate all such transformations.
Recalling that $\rob$ specifically considers the MLP layers contained within two models, we attempt to fool it by randomly reinitializing and retraining these MLP layers individually from scratch.
Our motivation is to simulate a somewhat strong adversary. We also conjecture that if there exist transformations of the weights of either MLP that preserve model output but fool our test, then retraining from scratch will be as likely to find these transformed weights versus something close to the original; thus, if our test is robust to retraining MLP layers from scratch then this suggests it may be robust to a broad variety of such transformations.

We reinitialize the first GLU MLP submodel of a model $\theta_1$ with an MLP with double the width and train it to minimize mean squared error with respect to the outputs of the original MLP, over an isotropic Gaussian input distribution.
We retrain each of the 32 MLP layers (keeping other layers fixed) of \texttt{vicuna-7b-v1.5} (a finetune of \texttt{Llama-2-7b-hf}) for 10k gradient steps (until the loss curve plateus). (Additional hyperparameters and a learning curve are in Appendix \ref{app:retrain}.) For all 32 runs, we compute $\rob$ for the retrained model with the original \texttt{Llama-2-7b-hf} and find $\rob$ remains very small between two non-independent models even if one model's MLP blocks have been retrained from scratch.
For example, retraining the first MLP layer (with a final train loss of 0.0048), the value of the statistic $\rob^{(1)}$ on the first MLP was less than $\varepsilon = $ 2.2e-308, indicating that the two models are not independent. We find the same is true for the other MLP layers as well (i.e. $\rob^{(\ell)}$ when evaluated on retrained layer $\ell$), with full results in Table \ref{tab:mlp_retrain} of Appendix \ref{app:retrain}.

\subsubsection{Generalizing to different architectures}
\label{app:distill-glu}
As we describe in Section~\ref{sec:generalized}, we can also apply our test to model architectures which do not contain GLU MLP submodels.
For example, the GPT-2 architecture uses a standard 2-layer MLP rather than a GLU MLP.
We apply our test (Algorithm~\ref{algorithm:general-test}) to \texttt{GPT2\_PMC} and \texttt{gpt2},
where the former is a finetune of the latter \citep{radford2019language}.
We use 30k training steps with an isotropic Gaussian input distribution to learn the GLU MLP parameters with which we replace the original MLP submodels in each model.
The test yields a value of 3.034e-61, thus distinguishing the two models as dependent. We show additional results on independent and non-independent models in Table \ref{tab:distill_mlp}.

\begin{table}[]
    \centering
    \begin{tabular}{c|c|c|c}
    \hline 
        $\theta_1$ & $\theta_2$ & Indep.? & $\rob^{(1)}$ \\ \hline
        \texttt{gpt2} & \texttt{GPT2\_PMC} & \ding{55} & 3.034e-61 \\ 
        \texttt{gpt2} & \texttt{artgpt2tox} & \ding{55} & 1.049e-75 \\
        \texttt{gpt2} & \texttt{distilgpt2} & \ding{55} & 1.079e-63 \\ 
        \texttt{Llama-3.2-1B} & \texttt{Llama-3.2-3B} & \ding{55} & 2.011e-70 \\ \hline 
        \texttt{openai-gpt} & \texttt{gpt2} & \ding{51} & 0.359 \\
        \texttt{openai-gpt} & \texttt{distilgpt2} & \ding{51} & 0.770 \\
        \texttt{gpt} & \texttt{Llama-3.2-1B} & \ding{51} & 0.481 \\ \hline 
    \end{tabular}
    \caption{We evaluate $\rob^{(1)}$ on models distilled with a GLU MLP. For GPT models, we reinitialize the feedforward-network with a GLU MLP and use Algorithm \ref{algorithm:general-test}. We find that even after training a GLU MLP, our unconstrained statistic $\rob$ detects non-independent models with low values, but not for independent models (indicated by the ground truth ``Indep.?'' column).}
    \label{tab:distill_mlp}
\end{table}

\subsection{Localized testing}
\label{sec:localized-testing}

Finally, we use $\rob$ on models pairs with different hidden dimensions, specifically on 
pruned model pairs, when model dimensions are reduced by preserving only select weights. 

In particular, we identify the specific Transformer blocks of $\texttt{Llama-3.1-8B}$ whose weights were likely used in initializing $\texttt{Llama-3.2-3B}$ and $\texttt{Llama-3.2-1B}$, as Meta reported that the first two models were pruned from the third \citep{llamablog}. We use $\rob$ on all pairs of MLP blocks (i.e. $\rob^{(i,j)}$ in equation (\ref{eqn:matching-blocks})), 
and match blocks by identifying pairs, $i$ from $\firstparam$ and $j$ from $\secondparam$, such that $\rob^{(i,j)}$ is less than 1e-4. We report the matched layers between the Llama 3.1 and Llama 3.2 models in Figure \ref{fig:llama-match} and in Appendix \ref{app:modelblockmatching}.

\begin{figure}[h]
    \centering
    \includegraphics[width=0.95\linewidth]{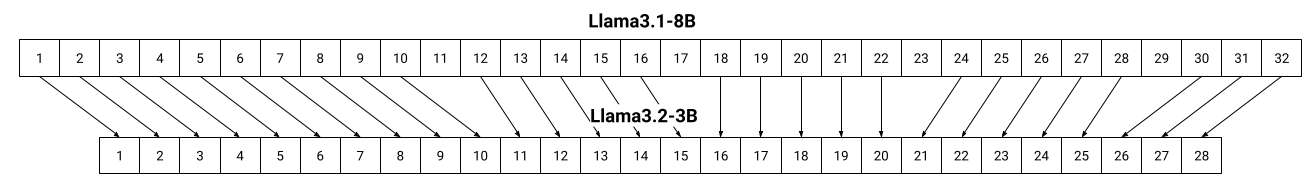}
    \caption{We evaluate $\rob^{(i,j)}$, the unconstrained setting statistic, between all pairs of GLU MLPs in Transformer block $i \in \{1, 2, \dots, 32 \}$ of Llama 3.1-8B and Transformer block $j \in \{1, 2, \dots, 28 \}$ of Llama 3.2-3B. Arrows indicate if $\rob^{(i,j)} < $ 1e-4 and suggest which Transformer blocks of Llama 3.1-8B were kept in the pruning process to initialize Llama 3.2-3B.}
    \label{fig:llama-match}
\end{figure}

We also identify which hidden units were most likely shared between the blocks when MLP dimension is reduced (from 14336 to 8192) during pruning, from the permutation $\pi$ returned from $\speartest(H_\text{up}^{(\ell)}(\theta_1),H_\text{up}^{(\ell)}(\theta_2))$---as $\pi$ describes the $h_2$ rows of $H_\text{up}^{(\ell)}(\theta_1)$ that are most similar with the rows of $H_\text{up}^{(\ell)}(\theta_2)$, which potentially describe the hidden units preserved in $\firstparam$ as it is pruned to $\secondparam$. The plot in Figure \ref{fig:matched-activations-llama} shows the activation rows from the up projection matrix $U^{(1)}$ of the first MLP of Llama 3.2-3B (8192 rows) (on the x-axis) matched with the rows from the up projection matrix of the first MLP of Llama 3.1-8B (out of 14336 rows) (on the y-axis). In particular, we can see that the activations are not simply the first the first 8192 rows pruned from the 14336-dimensional MLP, rather they appear to be distributed across all 14336 rows.

\begin{figure}
    \centering
    \includegraphics[width=0.5\linewidth]{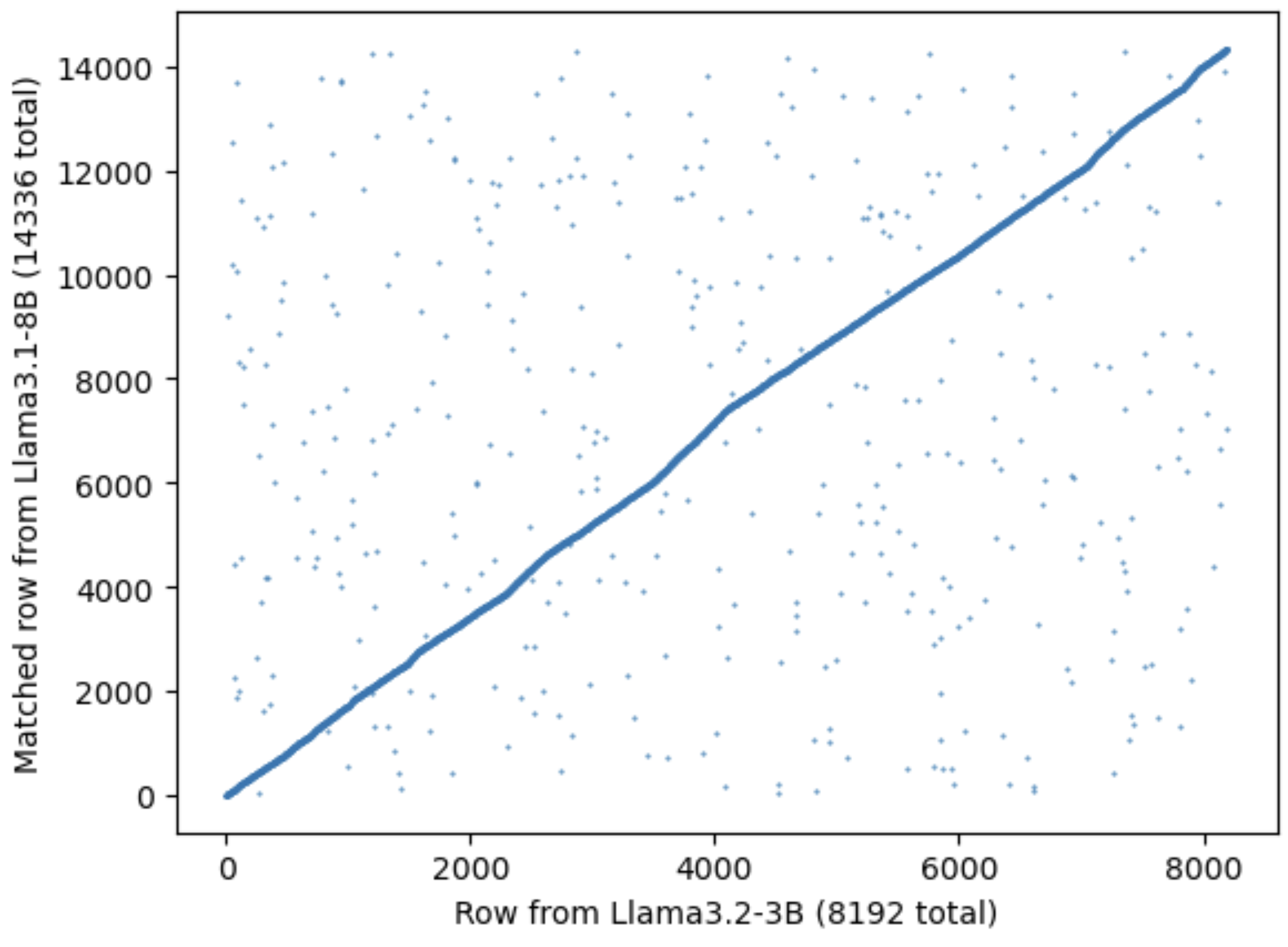}
    \caption{We align up-projection hidden activations from the first MLPs of Llama 3.1-8B and Llama 3.2-3B using $\speartest(H_\text{up}^{(\ell)}(\theta_1),H_\text{up}^{(\ell)}(\theta_2))$ and plot the activation row from Llama 3.2-3B on the x-axis and the matched activation row from Llama 3.1-8B on the y-axis. We see that the weights and activations of Llama 3.2-3B pruned from Llama 3.1-8B were likely uniformly selected.}
    \label{fig:matched-activations-llama}
\end{figure}

We also run this pairwise layer matching on the ShearedLlama \citep{xia2024shearedllamaacceleratinglanguage} models, which were the Llama 2-7B models pruned down to 1.3B and 2.7B parameters and find matching blocks, as well as on the pruned \texttt{Llama-3.1-Minitron-4B-Depth-Base} model \citep{muralidharan2024compactlanguagemodelspruning} from Llama 3.1, which we report in Appendix \ref{app:modelblockmatching}. 
\section{Related \& Future Work}\label{sec:related-work}

We develop methods for testing whether two models are independently trained given their weights that do not require intervening on the training of either model.

A related line of work known as model fingerprinting \citep{xu2024instructionalfingerprintinglargelanguage, zhang2024reefrepresentationencodingfingerprints, jin2024proflingofingerprintingbasedintellectualproperty, yang2024fingerprintlargelanguagemodels}
plants a secret signal in the weights of a model so that anyone who knows the key can detect the fingerprint from query access to the model (or fine-tunes of the model).
For example, \citet{xu2024instructionalfingerprintinglargelanguage} propose fingerprinting a model by fine-tuning on a secret random string; fingerprint detection then resolves to prompting a putative fingerprinted model with a prefix of the string.
Unlike \citet{xu2024instructionalfingerprintinglargelanguage}, we do not intervene on the training process of the models we test; however, we do require access to model weights in order to run our tests.

A separate line of work on text watermarking aims to attribute model-generated text by planting a watermark when sampling text from the model \citep{christ2023undetectablewatermarkslanguagemodels, kirchenbauer2024watermarklargelanguagemodels, kuditipudi2024robustdistortionfreewatermarkslanguage, aaronson2023watermarking, zhao2023provablerobustwatermarkingaigenerated}. 
Because it intervenes on sampling, text watermarking is inapplicable to open-weight models, the focus of both model fingerprinting and our setting. Recent work demonstrates that models can directly learn to generate watermarked text but also finds the learned watermark is not robust to further fine-tuning \citep{gu2024learnabilitywatermarkslanguagemodels}.

Finally, we show that the methods of the most related work \citep{zeng2024humanreadablefingerprintlargelanguage} (who investigate the same question of model provenance) are not robust to our transformations 
and also do not provide p-values for independence testing.
Other works like \citet{jin2024proflingofingerprintingbasedintellectualproperty} propose crafting specific queries that are likely to produce different responses among independently trained models and \cite{nikolic2025modelprovenancetestinglarge} identify similarities in non-independent model outputs; these methods do not require access to weights but also have fewer theoretical guarantees and do not produce exact p-values.

One limitation of our work is that we do not distinguish between different cases of dependent models, such as whether two models share a common ancestor versus one being a finetune of the other.
One direction for future work is to develop methods capable of differentiating between these cases to reconstruct a complete ``family tree'' of model lineage \citep{yax2024phylolminferringphylogeny}.
Another open question is whether it is possible to obtain exact guarantees (i.e., p-values) for our unconstrained setting.

\newpage 

\section*{Acknowledgments}
We gratefully acknowledge the support of this work by an NSF Frontier Award (NSF Grant no. 1805310) and Omidyar. Sally Zhu was supported by a Stanford CURIS Fellowship. Ahmed Ahmed is grateful to be supported by an NSF Graduate Research Fellowship and a Knight-Hennessy Fellowship.

\bibliography{biblio}

\begin{thebibliography}{35}
\providecommand{\natexlab}[1]{#1}
\providecommand{\url}[1]{\texttt{#1}}
\expandafter\ifx\csname urlstyle\endcsname\relax
  \providecommand{\doi}[1]{doi: #1}\else
  \providecommand{\doi}{doi: \begingroup \urlstyle{rm}\Url}\fi

\bibitem[Aaronson \& Kirchner(2023)Aaronson and Kirchner]{aaronson2023watermarking}
Scott Aaronson and Hendrik Kirchner.
\newblock {W}atermarking {G}{P}{T} {O}utputs, 2023.

\bibitem[Azerbayev et~al.(2024)Azerbayev, Schoelkopf, Paster, Santos, McAleer, Jiang, Deng, Biderman, and Welleck]{azerbayev2024llemmaopenlanguagemodel}
Zhangir Azerbayev, Hailey Schoelkopf, Keiran Paster, Marco~Dos Santos, Stephen~Marcus McAleer, Albert~Q. Jiang, Jia Deng, Stella Biderman, and Sean Welleck.
\newblock {Llemma: An Open Language Model for Mathematics}.
\newblock In \emph{The Twelfth International Conference on Learning Representations}, 2024.
\newblock URL \url{https://openreview.net/forum?id=4WnqRR915j}.

\bibitem[Chowdhery et~al.(2022)Chowdhery, Narang, Devlin, Bosma, Mishra, Roberts, Barham, Chung, Sutton, Gehrmann, Schuh, Shi, Tsvyashchenko, Maynez, Rao, Barnes, Tay, Shazeer, Prabhakaran, Reif, Du, Hutchinson, Pope, Bradbury, Austin, Isard, Gur-Ari, Yin, Duke, Levskaya, Ghemawat, Dev, Michalewski, Garcia, Misra, Robinson, Fedus, Zhou, Ippolito, Luan, Lim, Zoph, Spiridonov, Sepassi, Dohan, Agrawal, Omernick, Dai, Pillai, Pellat, Lewkowycz, Moreira, Child, Polozov, Lee, Zhou, Wang, Saeta, Diaz, Firat, Catasta, Wei, Meier-Hellstern, Eck, Dean, Petrov, and Fiedel]{chowdhery2022palmscalinglanguagemodeling}
Aakanksha Chowdhery, Sharan Narang, Jacob Devlin, Maarten Bosma, Gaurav Mishra, Adam Roberts, Paul Barham, Hyung~Won Chung, Charles Sutton, Sebastian Gehrmann, Parker Schuh, Kensen Shi, Sasha Tsvyashchenko, Joshua Maynez, Abhishek Rao, Parker Barnes, Yi~Tay, Noam Shazeer, Vinodkumar Prabhakaran, Emily Reif, Nan Du, Ben Hutchinson, Reiner Pope, James Bradbury, Jacob Austin, Michael Isard, Guy Gur-Ari, Pengcheng Yin, Toju Duke, Anselm Levskaya, Sanjay Ghemawat, Sunipa Dev, Henryk Michalewski, Xavier Garcia, Vedant Misra, Kevin Robinson, Liam Fedus, Denny Zhou, Daphne Ippolito, David Luan, Hyeontaek Lim, Barret Zoph, Alexander Spiridonov, Ryan Sepassi, David Dohan, Shivani Agrawal, Mark Omernick, Andrew~M. Dai, Thanumalayan~Sankaranarayana Pillai, Marie Pellat, Aitor Lewkowycz, Erica Moreira, Rewon Child, Oleksandr Polozov, Katherine Lee, Zongwei Zhou, Xuezhi Wang, Brennan Saeta, Mark Diaz, Orhan Firat, Michele Catasta, Jason Wei, Kathy Meier-Hellstern, Douglas Eck, Jeff Dean, Slav Petrov, and Noah Fiedel.
\newblock Palm: Scaling language modeling with pathways, 2022.
\newblock URL \url{https://arxiv.org/abs/2204.02311}.

\bibitem[Christ et~al.(2024)Christ, Gunn, and Zamir]{christ2023undetectablewatermarkslanguagemodels}
Miranda Christ, Sam Gunn, and Or~Zamir.
\newblock {Undetectable Watermarks for Language Models}.
\newblock In Shipra Agrawal and Aaron Roth (eds.), \emph{Proceedings of Thirty Seventh Conference on Learning Theory}, volume 247 of \emph{Proceedings of Machine Learning Research}, pp.\  1125--1139. PMLR, 30 Jun--03 Jul 2024.
\newblock URL \url{https://proceedings.mlr.press/v247/christ24a.html}.

\bibitem[DeepSeek-AI et~al.(2024)DeepSeek-AI, Liu, Feng, Xue, Wang, Wu, Lu, Zhao, Deng, Zhang, Ruan, Dai, Guo, Yang, Chen, Ji, Li, Lin, Dai, Luo, Hao, Chen, Li, Zhang, Bao, Xu, Wang, Zhang, Ding, Xin, Gao, Li, Qu, Cai, Liang, Guo, Ni, Li, Wang, Chen, Chen, Yuan, Qiu, Li, Song, Dong, Hu, Gao, Guan, Huang, Yu, Wang, Zhang, Xu, Xia, Zhao, Wang, Zhang, Li, Wang, Zhang, Zhang, Tang, Li, Tian, Huang, Wang, Zhang, Wang, Zhu, Chen, Du, Chen, Jin, Ge, Zhang, Pan, Wang, Xu, Zhang, Chen, Li, Lu, Zhou, Chen, Wu, Ye, Ye, Ma, Wang, Zhou, Yu, Zhou, Pan, Wang, Yun, Pei, Sun, Xiao, Zeng, Zhao, An, Liu, Liang, Gao, Yu, Zhang, Li, Jin, Wang, Bi, Liu, Wang, Shen, Chen, Zhang, Chen, Nie, Sun, Wang, Cheng, Liu, Xie, Liu, Yu, Song, Shan, Zhou, Yang, Li, Su, Lin, Li, Wang, Wei, Zhu, Zhang, Xu, Xu, Huang, Li, Zhao, Sun, Li, Wang, Yu, Zheng, Zhang, Shi, Xiong, He, Tang, Piao, Wang, Tan, Ma, Liu, Guo, Wu, Ou, Zhu, Wang, Gong, Zou, He, Zha, Xiong, Ma, Yan, Luo, You, Liu, Zhou, Wu, Ren, Ren, Sha, Fu, Xu, Huang, Zhang, Xie, Zhang, Hao,
  Gou, Ma, Yan, Shao, Xu, Wu, Zhang, Li, Gu, Zhu, Liu, Li, Xie, Song, Gao, and Pan]{deepseekai2024deepseekv3technicalreport}
DeepSeek-AI, Aixin Liu, Bei Feng, Bing Xue, Bingxuan Wang, Bochao Wu, Chengda Lu, Chenggang Zhao, Chengqi Deng, Chenyu Zhang, Chong Ruan, Damai Dai, Daya Guo, Dejian Yang, Deli Chen, Dongjie Ji, Erhang Li, Fangyun Lin, Fucong Dai, Fuli Luo, Guangbo Hao, Guanting Chen, Guowei Li, H.~Zhang, Han Bao, Hanwei Xu, Haocheng Wang, Haowei Zhang, Honghui Ding, Huajian Xin, Huazuo Gao, Hui Li, Hui Qu, J.~L. Cai, Jian Liang, Jianzhong Guo, Jiaqi Ni, Jiashi Li, Jiawei Wang, Jin Chen, Jingchang Chen, Jingyang Yuan, Junjie Qiu, Junlong Li, Junxiao Song, Kai Dong, Kai Hu, Kaige Gao, Kang Guan, Kexin Huang, Kuai Yu, Lean Wang, Lecong Zhang, Lei Xu, Leyi Xia, Liang Zhao, Litong Wang, Liyue Zhang, Meng Li, Miaojun Wang, Mingchuan Zhang, Minghua Zhang, Minghui Tang, Mingming Li, Ning Tian, Panpan Huang, Peiyi Wang, Peng Zhang, Qiancheng Wang, Qihao Zhu, Qinyu Chen, Qiushi Du, R.~J. Chen, R.~L. Jin, Ruiqi Ge, Ruisong Zhang, Ruizhe Pan, Runji Wang, Runxin Xu, Ruoyu Zhang, Ruyi Chen, S.~S. Li, Shanghao Lu, Shangyan Zhou, Shanhuang
  Chen, Shaoqing Wu, Shengfeng Ye, Shengfeng Ye, Shirong Ma, Shiyu Wang, Shuang Zhou, Shuiping Yu, Shunfeng Zhou, Shuting Pan, T.~Wang, Tao Yun, Tian Pei, Tianyu Sun, W.~L. Xiao, Wangding Zeng, Wanjia Zhao, Wei An, Wen Liu, Wenfeng Liang, Wenjun Gao, Wenqin Yu, Wentao Zhang, X.~Q. Li, Xiangyue Jin, Xianzu Wang, Xiao Bi, Xiaodong Liu, Xiaohan Wang, Xiaojin Shen, Xiaokang Chen, Xiaokang Zhang, Xiaosha Chen, Xiaotao Nie, Xiaowen Sun, Xiaoxiang Wang, Xin Cheng, Xin Liu, Xin Xie, Xingchao Liu, Xingkai Yu, Xinnan Song, Xinxia Shan, Xinyi Zhou, Xinyu Yang, Xinyuan Li, Xuecheng Su, Xuheng Lin, Y.~K. Li, Y.~Q. Wang, Y.~X. Wei, Y.~X. Zhu, Yang Zhang, Yanhong Xu, Yanhong Xu, Yanping Huang, Yao Li, Yao Zhao, Yaofeng Sun, Yaohui Li, Yaohui Wang, Yi~Yu, Yi~Zheng, Yichao Zhang, Yifan Shi, Yiliang Xiong, Ying He, Ying Tang, Yishi Piao, Yisong Wang, Yixuan Tan, Yiyang Ma, Yiyuan Liu, Yongqiang Guo, Yu~Wu, Yuan Ou, Yuchen Zhu, Yuduan Wang, Yue Gong, Yuheng Zou, Yujia He, Yukun Zha, Yunfan Xiong, Yunxian Ma, Yuting Yan, Yuxiang
  Luo, Yuxiang You, Yuxuan Liu, Yuyang Zhou, Z.~F. Wu, Z.~Z. Ren, Zehui Ren, Zhangli Sha, Zhe Fu, Zhean Xu, Zhen Huang, Zhen Zhang, Zhenda Xie, Zhengyan Zhang, Zhewen Hao, Zhibin Gou, Zhicheng Ma, Zhigang Yan, Zhihong Shao, Zhipeng Xu, Zhiyu Wu, Zhongyu Zhang, Zhuoshu Li, Zihui Gu, Zijia Zhu, Zijun Liu, Zilin Li, Ziwei Xie, Ziyang Song, Ziyi Gao, and Zizheng Pan.
\newblock Deepseek-v3 technical report, 2024.
\newblock URL \url{https://arxiv.org/abs/2412.19437}.

\bibitem[Dubey et~al.(2024)Dubey, Jauhri, Pandey, Kadian, Al-Dahle, Letman, Mathur, Schelten, Yang, Fan, Goyal, Hartshorn, Yang, Mitra, Sravankumar, Korenev, Hinsvark, Rao, Zhang, Rodriguez, Gregerson, Spataru, Roziere, Biron, Tang, Chern, Caucheteux, Nayak, Bi, Marra, McConnell, Keller, Touret, Wu, Wong, Ferrer, Nikolaidis, Allonsius, Song, Pintz, Livshits, Esiobu, Choudhary, Mahajan, Garcia-Olano, Perino, Hupkes, Lakomkin, AlBadawy, Lobanova, Dinan, Smith, Radenovic, Zhang, Synnaeve, Lee, Anderson, Nail, Mialon, Pang, Cucurell, Nguyen, Korevaar, Xu, Touvron, Zarov, Ibarra, Kloumann, Misra, Evtimov, Copet, Lee, Geffert, Vranes, Park, Mahadeokar, Shah, van~der Linde, Billock, Hong, Lee, Fu, Chi, Huang, Liu, Wang, Yu, Bitton, Spisak, Park, Rocca, Johnstun, Saxe, Jia, Alwala, Upasani, Plawiak, Li, Heafield, Stone, El-Arini, Iyer, Malik, Chiu, Bhalla, Rantala-Yeary, van~der Maaten, Chen, Tan, Jenkins, Martin, Madaan, Malo, Blecher, Landzaat, de~Oliveira, Muzzi, Pasupuleti, Singh, Paluri, Kardas, Oldham, Rita,
  Pavlova, Kambadur, Lewis, Si, Singh, Hassan, Goyal, Torabi, Bashlykov, Bogoychev, Chatterji, Duchenne, Çelebi, Alrassy, Zhang, Li, Vasic, Weng, Bhargava, Dubal, Krishnan, Koura, Xu, He, Dong, Srinivasan, Ganapathy, Calderer, Cabral, Stojnic, Raileanu, Girdhar, Patel, Sauvestre, Polidoro, Sumbaly, Taylor, Silva, Hou, Wang, Hosseini, Chennabasappa, Singh, Bell, Kim, Edunov, Nie, Narang, Raparthy, Shen, Wan, Bhosale, Zhang, Vandenhende, Batra, Whitman, Sootla, Collot, Gururangan, Borodinsky, Herman, Fowler, Sheasha, Georgiou, Scialom, Speckbacher, Mihaylov, Xiao, Karn, Goswami, Gupta, Ramanathan, Kerkez, Gonguet, Do, Vogeti, Petrovic, Chu, Xiong, Fu, Meers, Martinet, Wang, Tan, Xie, Jia, Wang, Goldschlag, Gaur, Babaei, Wen, Song, Zhang, Li, Mao, Coudert, Yan, Chen, Papakipos, Singh, Grattafiori, Jain, Kelsey, Shajnfeld, Gangidi, Victoria, Goldstand, Menon, Sharma, Boesenberg, Vaughan, Baevski, Feinstein, Kallet, Sangani, Yunus, Lupu, Alvarado, Caples, Gu, Ho, Poulton, Ryan, Ramchandani, Franco, Saraf,
  Chowdhury, Gabriel, Bharambe, Eisenman, Yazdan, James, Maurer, Leonhardi, Huang, Loyd, Paola, Paranjape, Liu, Wu, Ni, Hancock, Wasti, Spence, Stojkovic, Gamido, Montalvo, Parker, Burton, Mejia, Wang, Kim, Zhou, Hu, Chu, Cai, Tindal, Feichtenhofer, Civin, Beaty, Kreymer, Li, Wyatt, Adkins, Xu, Testuggine, David, Parikh, Liskovich, Foss, Wang, Le, Holland, Dowling, Jamil, Montgomery, Presani, Hahn, Wood, Brinkman, Arcaute, Dunbar, Smothers, Sun, Kreuk, Tian, Ozgenel, Caggioni, Guzmán, Kanayet, Seide, Florez, Schwarz, Badeer, Swee, Halpern, Thattai, Herman, Sizov, Guangyi, Zhang, Lakshminarayanan, Shojanazeri, Zou, Wang, Zha, Habeeb, Rudolph, Suk, Aspegren, Goldman, Molybog, Tufanov, Veliche, Gat, Weissman, Geboski, Kohli, Asher, Gaya, Marcus, Tang, Chan, Zhen, Reizenstein, Teboul, Zhong, Jin, Yang, Cummings, Carvill, Shepard, McPhie, Torres, Ginsburg, Wang, Wu, U, Saxena, Prasad, Khandelwal, Zand, Matosich, Veeraraghavan, Michelena, Li, Huang, Chawla, Lakhotia, Huang, Chen, Garg, A, Silva, Bell, Zhang, Guo,
  Yu, Moshkovich, Wehrstedt, Khabsa, Avalani, Bhatt, Tsimpoukelli, Mankus, Hasson, Lennie, Reso, Groshev, Naumov, Lathi, Keneally, Seltzer, Valko, Restrepo, Patel, Vyatskov, Samvelyan, Clark, Macey, Wang, Hermoso, Metanat, Rastegari, Bansal, Santhanam, Parks, White, Bawa, Singhal, Egebo, Usunier, Laptev, Dong, Zhang, Cheng, Chernoguz, Hart, Salpekar, Kalinli, Kent, Parekh, Saab, Balaji, Rittner, Bontrager, Roux, Dollar, Zvyagina, Ratanchandani, Yuvraj, Liang, Alao, Rodriguez, Ayub, Murthy, Nayani, Mitra, Li, Hogan, Battey, Wang, Maheswari, Howes, Rinott, Bondu, Datta, Chugh, Hunt, Dhillon, Sidorov, Pan, Verma, Yamamoto, Ramaswamy, Lindsay, Lindsay, Feng, Lin, Zha, Shankar, Zhang, Zhang, Wang, Agarwal, Sajuyigbe, Chintala, Max, Chen, Kehoe, Satterfield, Govindaprasad, Gupta, Cho, Virk, Subramanian, Choudhury, Goldman, Remez, Glaser, Best, Kohler, Robinson, Li, Zhang, Matthews, Chou, Shaked, Vontimitta, Ajayi, Montanez, Mohan, Kumar, Mangla, Ionescu, Poenaru, Mihailescu, Ivanov, Li, Wang, Jiang, Bouaziz,
  Constable, Tang, Wang, Wu, Wang, Xia, Wu, Gao, Chen, Hu, Jia, Qi, Li, Zhang, Zhang, Adi, Nam, Yu, Wang, Hao, Qian, He, Rait, DeVito, Rosnbrick, Wen, Yang, and Zhao]{dubey2024llama3herdmodels}
Abhimanyu Dubey, Abhinav Jauhri, Abhinav Pandey, Abhishek Kadian, Ahmad Al-Dahle, Aiesha Letman, Akhil Mathur, Alan Schelten, Amy Yang, Angela Fan, Anirudh Goyal, Anthony Hartshorn, Aobo Yang, Archi Mitra, Archie Sravankumar, Artem Korenev, Arthur Hinsvark, Arun Rao, Aston Zhang, Aurelien Rodriguez, Austen Gregerson, Ava Spataru, Baptiste Roziere, Bethany Biron, Binh Tang, Bobbie Chern, Charlotte Caucheteux, Chaya Nayak, Chloe Bi, Chris Marra, Chris McConnell, Christian Keller, Christophe Touret, Chunyang Wu, Corinne Wong, Cristian~Canton Ferrer, Cyrus Nikolaidis, Damien Allonsius, Daniel Song, Danielle Pintz, Danny Livshits, David Esiobu, Dhruv Choudhary, Dhruv Mahajan, Diego Garcia-Olano, Diego Perino, Dieuwke Hupkes, Egor Lakomkin, Ehab AlBadawy, Elina Lobanova, Emily Dinan, Eric~Michael Smith, Filip Radenovic, Frank Zhang, Gabriel Synnaeve, Gabrielle Lee, Georgia~Lewis Anderson, Graeme Nail, Gregoire Mialon, Guan Pang, Guillem Cucurell, Hailey Nguyen, Hannah Korevaar, Hu~Xu, Hugo Touvron, Iliyan Zarov,
  Imanol~Arrieta Ibarra, Isabel Kloumann, Ishan Misra, Ivan Evtimov, Jade Copet, Jaewon Lee, Jan Geffert, Jana Vranes, Jason Park, Jay Mahadeokar, Jeet Shah, Jelmer van~der Linde, Jennifer Billock, Jenny Hong, Jenya Lee, Jeremy Fu, Jianfeng Chi, Jianyu Huang, Jiawen Liu, Jie Wang, Jiecao Yu, Joanna Bitton, Joe Spisak, Jongsoo Park, Joseph Rocca, Joshua Johnstun, Joshua Saxe, Junteng Jia, Kalyan~Vasuden Alwala, Kartikeya Upasani, Kate Plawiak, Ke~Li, Kenneth Heafield, Kevin Stone, Khalid El-Arini, Krithika Iyer, Kshitiz Malik, Kuenley Chiu, Kunal Bhalla, Lauren Rantala-Yeary, Laurens van~der Maaten, Lawrence Chen, Liang Tan, Liz Jenkins, Louis Martin, Lovish Madaan, Lubo Malo, Lukas Blecher, Lukas Landzaat, Luke de~Oliveira, Madeline Muzzi, Mahesh Pasupuleti, Mannat Singh, Manohar Paluri, Marcin Kardas, Mathew Oldham, Mathieu Rita, Maya Pavlova, Melanie Kambadur, Mike Lewis, Min Si, Mitesh~Kumar Singh, Mona Hassan, Naman Goyal, Narjes Torabi, Nikolay Bashlykov, Nikolay Bogoychev, Niladri Chatterji, Olivier
  Duchenne, Onur Çelebi, Patrick Alrassy, Pengchuan Zhang, Pengwei Li, Petar Vasic, Peter Weng, Prajjwal Bhargava, Pratik Dubal, Praveen Krishnan, Punit~Singh Koura, Puxin Xu, Qing He, Qingxiao Dong, Ragavan Srinivasan, Raj Ganapathy, Ramon Calderer, Ricardo~Silveira Cabral, Robert Stojnic, Roberta Raileanu, Rohit Girdhar, Rohit Patel, Romain Sauvestre, Ronnie Polidoro, Roshan Sumbaly, Ross Taylor, Ruan Silva, Rui Hou, Rui Wang, Saghar Hosseini, Sahana Chennabasappa, Sanjay Singh, Sean Bell, Seohyun~Sonia Kim, Sergey Edunov, Shaoliang Nie, Sharan Narang, Sharath Raparthy, Sheng Shen, Shengye Wan, Shruti Bhosale, Shun Zhang, Simon Vandenhende, Soumya Batra, Spencer Whitman, Sten Sootla, Stephane Collot, Suchin Gururangan, Sydney Borodinsky, Tamar Herman, Tara Fowler, Tarek Sheasha, Thomas Georgiou, Thomas Scialom, Tobias Speckbacher, Todor Mihaylov, Tong Xiao, Ujjwal Karn, Vedanuj Goswami, Vibhor Gupta, Vignesh Ramanathan, Viktor Kerkez, Vincent Gonguet, Virginie Do, Vish Vogeti, Vladan Petrovic, Weiwei Chu,
  Wenhan Xiong, Wenyin Fu, Whitney Meers, Xavier Martinet, Xiaodong Wang, Xiaoqing~Ellen Tan, Xinfeng Xie, Xuchao Jia, Xuewei Wang, Yaelle Goldschlag, Yashesh Gaur, Yasmine Babaei, Yi~Wen, Yiwen Song, Yuchen Zhang, Yue Li, Yuning Mao, Zacharie~Delpierre Coudert, Zheng Yan, Zhengxing Chen, Zoe Papakipos, Aaditya Singh, Aaron Grattafiori, Abha Jain, Adam Kelsey, Adam Shajnfeld, Adithya Gangidi, Adolfo Victoria, Ahuva Goldstand, Ajay Menon, Ajay Sharma, Alex Boesenberg, Alex Vaughan, Alexei Baevski, Allie Feinstein, Amanda Kallet, Amit Sangani, Anam Yunus, Andrei Lupu, Andres Alvarado, Andrew Caples, Andrew Gu, Andrew Ho, Andrew Poulton, Andrew Ryan, Ankit Ramchandani, Annie Franco, Aparajita Saraf, Arkabandhu Chowdhury, Ashley Gabriel, Ashwin Bharambe, Assaf Eisenman, Azadeh Yazdan, Beau James, Ben Maurer, Benjamin Leonhardi, Bernie Huang, Beth Loyd, Beto~De Paola, Bhargavi Paranjape, Bing Liu, Bo~Wu, Boyu Ni, Braden Hancock, Bram Wasti, Brandon Spence, Brani Stojkovic, Brian Gamido, Britt Montalvo, Carl
  Parker, Carly Burton, Catalina Mejia, Changhan Wang, Changkyu Kim, Chao Zhou, Chester Hu, Ching-Hsiang Chu, Chris Cai, Chris Tindal, Christoph Feichtenhofer, Damon Civin, Dana Beaty, Daniel Kreymer, Daniel Li, Danny Wyatt, David Adkins, David Xu, Davide Testuggine, Delia David, Devi Parikh, Diana Liskovich, Didem Foss, Dingkang Wang, Duc Le, Dustin Holland, Edward Dowling, Eissa Jamil, Elaine Montgomery, Eleonora Presani, Emily Hahn, Emily Wood, Erik Brinkman, Esteban Arcaute, Evan Dunbar, Evan Smothers, Fei Sun, Felix Kreuk, Feng Tian, Firat Ozgenel, Francesco Caggioni, Francisco Guzmán, Frank Kanayet, Frank Seide, Gabriela~Medina Florez, Gabriella Schwarz, Gada Badeer, Georgia Swee, Gil Halpern, Govind Thattai, Grant Herman, Grigory Sizov, Guangyi, Zhang, Guna Lakshminarayanan, Hamid Shojanazeri, Han Zou, Hannah Wang, Hanwen Zha, Haroun Habeeb, Harrison Rudolph, Helen Suk, Henry Aspegren, Hunter Goldman, Igor Molybog, Igor Tufanov, Irina-Elena Veliche, Itai Gat, Jake Weissman, James Geboski, James Kohli,
  Japhet Asher, Jean-Baptiste Gaya, Jeff Marcus, Jeff Tang, Jennifer Chan, Jenny Zhen, Jeremy Reizenstein, Jeremy Teboul, Jessica Zhong, Jian Jin, Jingyi Yang, Joe Cummings, Jon Carvill, Jon Shepard, Jonathan McPhie, Jonathan Torres, Josh Ginsburg, Junjie Wang, Kai Wu, Kam~Hou U, Karan Saxena, Karthik Prasad, Kartikay Khandelwal, Katayoun Zand, Kathy Matosich, Kaushik Veeraraghavan, Kelly Michelena, Keqian Li, Kun Huang, Kunal Chawla, Kushal Lakhotia, Kyle Huang, Lailin Chen, Lakshya Garg, Lavender A, Leandro Silva, Lee Bell, Lei Zhang, Liangpeng Guo, Licheng Yu, Liron Moshkovich, Luca Wehrstedt, Madian Khabsa, Manav Avalani, Manish Bhatt, Maria Tsimpoukelli, Martynas Mankus, Matan Hasson, Matthew Lennie, Matthias Reso, Maxim Groshev, Maxim Naumov, Maya Lathi, Meghan Keneally, Michael~L. Seltzer, Michal Valko, Michelle Restrepo, Mihir Patel, Mik Vyatskov, Mikayel Samvelyan, Mike Clark, Mike Macey, Mike Wang, Miquel~Jubert Hermoso, Mo~Metanat, Mohammad Rastegari, Munish Bansal, Nandhini Santhanam, Natascha
  Parks, Natasha White, Navyata Bawa, Nayan Singhal, Nick Egebo, Nicolas Usunier, Nikolay~Pavlovich Laptev, Ning Dong, Ning Zhang, Norman Cheng, Oleg Chernoguz, Olivia Hart, Omkar Salpekar, Ozlem Kalinli, Parkin Kent, Parth Parekh, Paul Saab, Pavan Balaji, Pedro Rittner, Philip Bontrager, Pierre Roux, Piotr Dollar, Polina Zvyagina, Prashant Ratanchandani, Pritish Yuvraj, Qian Liang, Rachad Alao, Rachel Rodriguez, Rafi Ayub, Raghotham Murthy, Raghu Nayani, Rahul Mitra, Raymond Li, Rebekkah Hogan, Robin Battey, Rocky Wang, Rohan Maheswari, Russ Howes, Ruty Rinott, Sai~Jayesh Bondu, Samyak Datta, Sara Chugh, Sara Hunt, Sargun Dhillon, Sasha Sidorov, Satadru Pan, Saurabh Verma, Seiji Yamamoto, Sharadh Ramaswamy, Shaun Lindsay, Shaun Lindsay, Sheng Feng, Shenghao Lin, Shengxin~Cindy Zha, Shiva Shankar, Shuqiang Zhang, Shuqiang Zhang, Sinong Wang, Sneha Agarwal, Soji Sajuyigbe, Soumith Chintala, Stephanie Max, Stephen Chen, Steve Kehoe, Steve Satterfield, Sudarshan Govindaprasad, Sumit Gupta, Sungmin Cho, Sunny
  Virk, Suraj Subramanian, Sy~Choudhury, Sydney Goldman, Tal Remez, Tamar Glaser, Tamara Best, Thilo Kohler, Thomas Robinson, Tianhe Li, Tianjun Zhang, Tim Matthews, Timothy Chou, Tzook Shaked, Varun Vontimitta, Victoria Ajayi, Victoria Montanez, Vijai Mohan, Vinay~Satish Kumar, Vishal Mangla, Vlad Ionescu, Vlad Poenaru, Vlad~Tiberiu Mihailescu, Vladimir Ivanov, Wei Li, Wenchen Wang, Wenwen Jiang, Wes Bouaziz, Will Constable, Xiaocheng Tang, Xiaofang Wang, Xiaojian Wu, Xiaolan Wang, Xide Xia, Xilun Wu, Xinbo Gao, Yanjun Chen, Ye~Hu, Ye~Jia, Ye~Qi, Yenda Li, Yilin Zhang, Ying Zhang, Yossi Adi, Youngjin Nam, Yu, Wang, Yuchen Hao, Yundi Qian, Yuzi He, Zach Rait, Zachary DeVito, Zef Rosnbrick, Zhaoduo Wen, Zhenyu Yang, and Zhiwei Zhao.
\newblock The {L}lama 3 {H}erd of {M}odels, 2024.
\newblock URL \url{https://arxiv.org/abs/2407.21783}.

\bibitem[Groeneveld et~al.(2024)Groeneveld, Beltagy, Walsh, Bhagia, Kinney, Tafjord, Jha, Ivison, Magnusson, Wang, Arora, Atkinson, Authur, Chandu, Cohan, Dumas, Elazar, Gu, Hessel, Khot, Merrill, Morrison, Muennighoff, Naik, Nam, Peters, Pyatkin, Ravichander, Schwenk, Shah, Smith, Strubell, Subramani, Wortsman, Dasigi, Lambert, Richardson, Zettlemoyer, Dodge, Lo, Soldaini, Smith, and Hajishirzi]{groeneveld2024olmoacceleratingsciencelanguage}
Dirk Groeneveld, Iz~Beltagy, Evan Walsh, Akshita Bhagia, Rodney Kinney, Oyvind Tafjord, Ananya Jha, Hamish Ivison, Ian Magnusson, Yizhong Wang, Shane Arora, David Atkinson, Russell Authur, Khyathi Chandu, Arman Cohan, Jennifer Dumas, Yanai Elazar, Yuling Gu, Jack Hessel, Tushar Khot, William Merrill, Jacob Morrison, Niklas Muennighoff, Aakanksha Naik, Crystal Nam, Matthew Peters, Valentina Pyatkin, Abhilasha Ravichander, Dustin Schwenk, Saurabh Shah, William Smith, Emma Strubell, Nishant Subramani, Mitchell Wortsman, Pradeep Dasigi, Nathan Lambert, Kyle Richardson, Luke Zettlemoyer, Jesse Dodge, Kyle Lo, Luca Soldaini, Noah Smith, and Hannaneh Hajishirzi.
\newblock {OLM}o: Accelerating the science of language models.
\newblock In Lun-Wei Ku, Andre Martins, and Vivek Srikumar (eds.), \emph{Proceedings of the 62nd Annual Meeting of the Association for Computational Linguistics (Volume 1: Long Papers)}, pp.\  15789--15809, Bangkok, Thailand, August 2024. Association for Computational Linguistics.
\newblock \doi{10.18653/v1/2024.acl-long.841}.
\newblock URL \url{https://aclanthology.org/2024.acl-long.841/}.

\bibitem[Gu et~al.(2024)Gu, Li, Liang, and Hashimoto]{gu2024learnabilitywatermarkslanguagemodels}
Chenchen Gu, Xiang~Lisa Li, Percy Liang, and Tatsunori Hashimoto.
\newblock {On the Learnability of Watermarks for Language Models}.
\newblock In \emph{The Twelfth International Conference on Learning Representations}, 2024.
\newblock URL \url{https://openreview.net/forum?id=9k0krNzvlV}.

\bibitem[Jin et~al.(2024)Jin, Zhang, Shi, Lou, and Hou]{jin2024proflingofingerprintingbasedintellectualproperty}
Heng Jin, Chaoyu Zhang, Shanghao Shi, Wenjing Lou, and Y.~Thomas Hou.
\newblock {ProFLingo: A Fingerprinting-based Intellectual Property Protection Scheme for Large Language Models}.
\newblock In \emph{2024 IEEE Conference on Communications and Network Security (CNS)}, pp.\  1--9, 2024.
\newblock \doi{10.1109/CNS62487.2024.10735575}.

\bibitem[Kapoor et~al.(2025)Kapoor, Bommasani, Klyman, Longpre, Ramaswami, Cihon, Hopkins, Bankston, Biderman, Bogen, Chowdhury, Engler, Henderson, Jernite, Lazar, Maffulli, Nelson, Pineau, Skowron, Song, Storchan, Zhang, Ho, Liang, and Narayanan]{kapoor2024openmodels}
Sayash Kapoor, Rishi Bommasani, Kevin Klyman, Shayne Longpre, Ashwin Ramaswami, Peter Cihon, Aspen Hopkins, Kevin Bankston, Stella Biderman, Miranda Bogen, Rumman Chowdhury, Alex Engler, Peter Henderson, Yacine Jernite, Seth Lazar, Stefano Maffulli, Alondra Nelson, Joelle Pineau, Aviya Skowron, Dawn Song, Victor Storchan, Daniel Zhang, Daniel~E. Ho, Percy Liang, and Arvind Narayanan.
\newblock {P}osition: {O}n the {S}ocietal {I}mpact of {O}pen {F}oundation {M}odels.
\newblock In \emph{Proceedings of the 41st International Conference on Machine Learning}, ICML'24. JMLR.org, 2025.

\bibitem[Kirchenbauer et~al.(2023)Kirchenbauer, Geiping, Wen, Katz, Miers, and Goldstein]{kirchenbauer2024watermarklargelanguagemodels}
John Kirchenbauer, Jonas Geiping, Yuxin Wen, Jonathan Katz, Ian Miers, and Tom Goldstein.
\newblock {A Watermark for Large Language Models}.
\newblock In Andreas Krause, Emma Brunskill, Kyunghyun Cho, Barbara Engelhardt, Sivan Sabato, and Jonathan Scarlett (eds.), \emph{Proceedings of the 40th International Conference on Machine Learning}, volume 202 of \emph{Proceedings of Machine Learning Research}, pp.\  17061--17084. PMLR, 23--29 Jul 2023.
\newblock URL \url{https://proceedings.mlr.press/v202/kirchenbauer23a.html}.

\bibitem[Kuditipudi et~al.(2024)Kuditipudi, Thickstun, Hashimoto, and Liang]{kuditipudi2024robustdistortionfreewatermarkslanguage}
Rohith Kuditipudi, John Thickstun, Tatsunori Hashimoto, and Percy Liang.
\newblock {Robust Distortion-free Watermarks for Language Models}.
\newblock \emph{Transactions on Machine Learning Research}, 2024.
\newblock ISSN 2835-8856.
\newblock URL \url{https://openreview.net/forum?id=FpaCL1MO2C}.

\bibitem[Li et~al.(2023)Li, Hammoud, Itani, Khizbullin, and Ghanem]{li2023camelcommunicativeagentsmind}
Guohao Li, Hasan Hammoud, Hani Itani, Dmitrii Khizbullin, and Bernard Ghanem.
\newblock {CAMEL: Communicative Agents for "Mind" Exploration of Large Language Model Society}.
\newblock In A.~Oh, T.~Naumann, A.~Globerson, K.~Saenko, M.~Hardt, and S.~Levine (eds.), \emph{Advances in Neural Information Processing Systems}, volume~36, pp.\  51991--52008. Curran Associates, Inc., 2023.
\newblock URL \url{https://proceedings.neurips.cc/paper_files/paper/2023/file/a3621ee907def47c1b952ade25c67698-Paper-Conference.pdf}.

\bibitem[Lin(2006)]{jsd}
J.~Lin.
\newblock Divergence measures based on the {S}hannon entropy.
\newblock \emph{IEEE Trans. Inf. Theor.}, 37\penalty0 (1):\penalty0 145–151, September 2006.
\newblock ISSN 0018-9448.
\newblock \doi{10.1109/18.61115}.
\newblock URL \url{https://doi.org/10.1109/18.61115}.

\bibitem[Liu et~al.(2024)Liu, Qiao, Neiswanger, Wang, Tan, Tao, Li, Wang, Sun, Pangarkar, Fan, Gu, Miller, Zhuang, He, Li, Koto, Tang, Ranjan, Shen, Iriondo, Mu, Hu, Schulze, Nakov, Baldwin, and Xing]{liu2023llm360}
Zhengzhong Liu, Aurick Qiao, Willie Neiswanger, Hongyi Wang, Bowen Tan, Tianhua Tao, Junbo Li, Yuqi Wang, Suqi Sun, Omkar Pangarkar, Richard Fan, Yi~Gu, Victor Miller, Yonghao Zhuang, Guowei He, Haonan Li, Fajri Koto, Liping Tang, Nikhil Ranjan, Zhiqiang Shen, Roberto Iriondo, Cun Mu, Zhiting Hu, Mark Schulze, Preslav Nakov, Timothy Baldwin, and Eric~P. Xing.
\newblock {{LLM}360: Towards Fully Transparent Open-Source {LLM}s}.
\newblock In \emph{First Conference on Language Modeling}, 2024.
\newblock URL \url{https://openreview.net/forum?id=QdWhj0QZFw}.

\bibitem[Mensch(2024)]{2024miquscandal}
Arthur Mensch.
\newblock {Mistral CEO confirms Miqu model leak}, August 2024.
\newblock URL \url{https://x.com/arthurmensch/status/1752737462663684344}.
\newblock Accessed: 2024-08-15.

\bibitem[MetaAI(2024)]{llamablog}
MetaAI.
\newblock {Llama 3.2: Revolutionizing edge AI and vision with open, customizable models}, 2024.
\newblock URL \url{https://ai.meta.com/blog/llama-3-2-connect-2024-vision-edge-mobile-devices/}.

\bibitem[Mosteller \& Fisher(1948)Mosteller and Fisher]{fisher_qa}
Frederick Mosteller and R.~A. Fisher.
\newblock {Questions and Answers}.
\newblock \emph{The American Statistician}, 2\penalty0 (5):\penalty0 30--31, 1948.
\newblock ISSN 00031305.
\newblock URL \url{http://www.jstor.org/stable/2681650}.

\bibitem[Muralidharan et~al.(2024)Muralidharan, Sreenivas, Joshi, Chochowski, Patwary, Shoeybi, Catanzaro, Kautz, and Molchanov]{muralidharan2024compactlanguagemodelspruning}
Saurav Muralidharan, Sharath~Turuvekere Sreenivas, Raviraj~Bhuminand Joshi, Marcin Chochowski, Mostofa Patwary, Mohammad Shoeybi, Bryan Catanzaro, Jan Kautz, and Pavlo Molchanov.
\newblock {Compact Language Models via Pruning and Knowledge Distillation}.
\newblock In \emph{The Thirty-eighth Annual Conference on Neural Information Processing Systems}, 2024.
\newblock URL \url{https://openreview.net/forum?id=9U0nLnNMJ7}.

\bibitem[Nikolic et~al.(2025)Nikolic, Baluta, and Saxena]{nikolic2025modelprovenancetestinglarge}
Ivica Nikolic, Teodora Baluta, and Prateek Saxena.
\newblock Model provenance testing for large language models, 2025.
\newblock URL \url{https://arxiv.org/abs/2502.00706}.

\bibitem[Peng et~al.(2023)Peng, Chen, Xu, et~al.]{Peng2023dnnip}
S.~Peng, Y.~Chen, J.~Xu, et~al.
\newblock Intellectual property protection of dnn models.
\newblock \emph{World Wide Web}, 26:\penalty0 1877--1911, July 2023.
\newblock \doi{10.1007/s11280-022-01113-3}.
\newblock URL \url{https://doi.org/10.1007/s11280-022-01113-3}.

\bibitem[Poli et~al.(2023)Poli, Wang, Massaroli, Quesnelle, Carlow, Nguyen, and Thomas]{stripedhyena}
Michael Poli, Jue Wang, Stefano Massaroli, Jeffrey Quesnelle, Ryan Carlow, Eric Nguyen, and Armin Thomas.
\newblock {StripedHyena: Moving Beyond Transformers with Hybrid Signal Processing Models}, 12 2023.
\newblock URL \url{https://github.com/togethercomputer/stripedhyena}.

\bibitem[Radford et~al.(2019)Radford, Wu, Child, Luan, Amodei, and Sutskever]{radford2019language}
Alec Radford, Jeff Wu, Rewon Child, David Luan, Dario Amodei, and Ilya Sutskever.
\newblock {Language Models are Unsupervised Multitask Learners}, 2019.
\newblock URL \url{https://cdn.openai.com/better-language-models/language_models_are_unsupervised_multitask_learners.pdf}.

\bibitem[Ramshaw \& Tarjan(2012)Ramshaw and Tarjan]{Ramshaw2012OnMA}
Lyle Ramshaw and Robert~Endre Tarjan.
\newblock {On Minimum-Cost Assignments in Unbalanced Bipartite Graphs}, 2012.
\newblock URL \url{https://api.semanticscholar.org/CorpusID:6964149}.

\bibitem[Soldaini et~al.(2024)Soldaini, Kinney, Bhagia, Schwenk, Atkinson, Authur, Bogin, Chandu, Dumas, Elazar, Hofmann, Jha, Kumar, Lucy, Lyu, Lambert, Magnusson, Morrison, Muennighoff, Naik, Nam, Peters, Ravichander, Richardson, Shen, Strubell, Subramani, Tafjord, Walsh, Zettlemoyer, Smith, Hajishirzi, Beltagy, Groeneveld, Dodge, and Lo]{soldaini2024dolmaopencorpustrillion}
Luca Soldaini, Rodney Kinney, Akshita Bhagia, Dustin Schwenk, David Atkinson, Russell Authur, Ben Bogin, Khyathi Chandu, Jennifer Dumas, Yanai Elazar, Valentin Hofmann, Ananya Jha, Sachin Kumar, Li~Lucy, Xinxi Lyu, Nathan Lambert, Ian Magnusson, Jacob Morrison, Niklas Muennighoff, Aakanksha Naik, Crystal Nam, Matthew Peters, Abhilasha Ravichander, Kyle Richardson, Zejiang Shen, Emma Strubell, Nishant Subramani, Oyvind Tafjord, Evan Walsh, Luke Zettlemoyer, Noah Smith, Hannaneh Hajishirzi, Iz~Beltagy, Dirk Groeneveld, Jesse Dodge, and Kyle Lo.
\newblock {Dolma: an Open Corpus of Three Trillion Tokens for Language Model Pretraining Research}.
\newblock In Lun-Wei Ku, Andre Martins, and Vivek Srikumar (eds.), \emph{Proceedings of the 62nd Annual Meeting of the Association for Computational Linguistics (Volume 1: Long Papers)}, pp.\  15725--15788, Bangkok, Thailand, August 2024. Association for Computational Linguistics.
\newblock \doi{10.18653/v1/2024.acl-long.840}.
\newblock URL \url{https://aclanthology.org/2024.acl-long.840/}.

\bibitem[Spearman(1904)]{spearmanrank}
C.~Spearman.
\newblock The proof and measurement of association between two things.
\newblock \emph{The American Journal of Psychology}, 15\penalty0 (1):\penalty0 72--101, 1904.
\newblock ISSN 00029556.
\newblock URL \url{http://www.jstor.org/stable/1412159}.

\bibitem[Sudalairaj et~al.(2024)Sudalairaj, Bhandwaldar, Pareja, Xu, Cox, and Srivastava]{sudalairaj2024lablargescalealignmentchatbots}
Shivchander Sudalairaj, Abhishek Bhandwaldar, Aldo Pareja, Kai Xu, David~D. Cox, and Akash Srivastava.
\newblock {LAB: Large-Scale Alignment for ChatBots}, 2024.
\newblock URL \url{https://arxiv.org/abs/2403.01081}.

\bibitem[Touvron et~al.(2023)Touvron, Martin, Stone, Albert, Almahairi, Babaei, Bashlykov, Batra, Bhargava, Bhosale, Bikel, Blecher, Ferrer, Chen, Cucurull, Esiobu, Fernandes, Fu, Fu, Fuller, Gao, Goswami, Goyal, Hartshorn, Hosseini, Hou, Inan, Kardas, Kerkez, Khabsa, Kloumann, Korenev, Koura, Lachaux, Lavril, Lee, Liskovich, Lu, Mao, Martinet, Mihaylov, Mishra, Molybog, Nie, Poulton, Reizenstein, Rungta, Saladi, Schelten, Silva, Smith, Subramanian, Tan, Tang, Taylor, Williams, Kuan, Xu, Yan, Zarov, Zhang, Fan, Kambadur, Narang, Rodriguez, Stojnic, Edunov, and Scialom]{touvron2023llama2openfoundation}
Hugo Touvron, Louis Martin, Kevin Stone, Peter Albert, Amjad Almahairi, Yasmine Babaei, Nikolay Bashlykov, Soumya Batra, Prajjwal Bhargava, Shruti Bhosale, Dan Bikel, Lukas Blecher, Cristian~Canton Ferrer, Moya Chen, Guillem Cucurull, David Esiobu, Jude Fernandes, Jeremy Fu, Wenyin Fu, Brian Fuller, Cynthia Gao, Vedanuj Goswami, Naman Goyal, Anthony Hartshorn, Saghar Hosseini, Rui Hou, Hakan Inan, Marcin Kardas, Viktor Kerkez, Madian Khabsa, Isabel Kloumann, Artem Korenev, Punit~Singh Koura, Marie-Anne Lachaux, Thibaut Lavril, Jenya Lee, Diana Liskovich, Yinghai Lu, Yuning Mao, Xavier Martinet, Todor Mihaylov, Pushkar Mishra, Igor Molybog, Yixin Nie, Andrew Poulton, Jeremy Reizenstein, Rashi Rungta, Kalyan Saladi, Alan Schelten, Ruan Silva, Eric~Michael Smith, Ranjan Subramanian, Xiaoqing~Ellen Tan, Binh Tang, Ross Taylor, Adina Williams, Jian~Xiang Kuan, Puxin Xu, Zheng Yan, Iliyan Zarov, Yuchen Zhang, Angela Fan, Melanie Kambadur, Sharan Narang, Aurelien Rodriguez, Robert Stojnic, Sergey Edunov, and Thomas
  Scialom.
\newblock Llama 2: {O}pen {F}oundation and {F}ine-{T}uned {C}hat {M}odels, 2023.
\newblock URL \url{https://arxiv.org/abs/2307.09288}.

\bibitem[Xia et~al.(2024)Xia, Gao, Zeng, and Chen]{xia2024shearedllamaacceleratinglanguage}
Mengzhou Xia, Tianyu Gao, Zhiyuan Zeng, and Danqi Chen.
\newblock {Sheared {LL}a{MA}: Accelerating Language Model Pre-training via Structured Pruning}.
\newblock In \emph{The Twelfth International Conference on Learning Representations}, 2024.
\newblock URL \url{https://openreview.net/forum?id=09iOdaeOzp}.

\bibitem[Xu et~al.(2024)Xu, Wang, Ma, Koh, Xiao, and Chen]{xu2024instructionalfingerprintinglargelanguage}
Jiashu Xu, Fei Wang, Mingyu Ma, Pang~Wei Koh, Chaowei Xiao, and Muhao Chen.
\newblock {Instructional Fingerprinting of Large Language Models}.
\newblock In Kevin Duh, Helena Gomez, and Steven Bethard (eds.), \emph{Proceedings of the 2024 Conference of the North American Chapter of the Association for Computational Linguistics: Human Language Technologies (Volume 1: Long Papers)}, pp.\  3277--3306, Mexico City, Mexico, June 2024. Association for Computational Linguistics.
\newblock \doi{10.18653/v1/2024.naacl-long.180}.
\newblock URL \url{https://aclanthology.org/2024.naacl-long.180/}.

\bibitem[Yang \& Wu(2024)Yang and Wu]{yang2024fingerprintlargelanguagemodels}
Zhiguang Yang and Hanzhou Wu.
\newblock A {F}ingerprint for {L}arge {L}anguage {M}odels, 2024.
\newblock URL \url{https://arxiv.org/abs/2407.01235}.

\bibitem[Yax et~al.(2024)Yax, Oudeyer, and Palminteri]{yax2024phylolminferringphylogeny}
Nicolas Yax, Pierre-Yves Oudeyer, and Stefano Palminteri.
\newblock {PhyloLM : Inferring the Phylogeny of Large Language Models and Predicting their Performances in Benchmarks}, 2024.
\newblock URL \url{https://arxiv.org/abs/2404.04671}.

\bibitem[Zeng et~al.(2024)Zeng, Zhou, Wang, and Lin]{zeng2024humanreadablefingerprintlargelanguage}
Boyi Zeng, Chenghu Zhou, Xinbing Wang, and Zhouhan Lin.
\newblock Human-{R}eadable {F}ingerprint for {L}arge {L}anguage {M}odels, 2024.
\newblock URL \url{https://arxiv.org/abs/2312.04828}.

\bibitem[Zhang et~al.(2024)Zhang, Liu, Qian, Zhang, Liu, Qiao, and Shao]{zhang2024reefrepresentationencodingfingerprints}
Jie Zhang, Dongrui Liu, Chen Qian, Linfeng Zhang, Yong Liu, Yu~Qiao, and Jing Shao.
\newblock {R}{E}{E}{F}: {R}epresentation {E}ncoding {F}ingerprints for {L}arge {L}anguage {M}odels, 2024.
\newblock URL \url{https://arxiv.org/abs/2410.14273}.

\bibitem[Zhao et~al.(2023)Zhao, Ananth, Li, and Wang]{zhao2023provablerobustwatermarkingaigenerated}
Xuandong Zhao, Prabhanjan Ananth, Lei Li, and Yu-Xiang Wang.
\newblock Provable robust watermarking for {AI}-generated text, 2023.
\newblock URL \url{https://arxiv.org/abs/2306.17439}.

\end{thebibliography}
\bibliographystyle{tmlr}

\appendix
\section{Randomized Learning Algorithms}\label{sec:randomized-alg}
\begin{definition}\label{defn:perm-equiv}
    Let $\Pi \subset \Theta \to \Theta$. Let $\pi \in \Pi$ and $\init \in \Theta$, with $\bar\theta \sim A(\init), \theta = \pi(\bar\theta)$ and $\theta' \sim A(\pi(\theta_0))$. A randomized learning algorithm $A : \Theta \to \mathcal{P}(\Theta)$ is $\Pi$-\textit{equivariant} if and only if $\theta \stackrel{d}{=} \theta'$.
\end{definition}

\begin{theorem}\label{thm:main-randomized}
    Let $\phi: \Theta \times \Theta \to \R$ be a test statistic and $\Pi \subset \Theta \to \Theta$ be finite. 
    Let $A : \Theta \to \mc{P}(\Theta)$ satisfy Definition~\ref{defn:perm-equiv} and let $P \in \mc{P}(\Theta)$ be $\Pi$-invariant. 
    For $\theta_1^0 \sim P$, let $\theta_1 \sim A(\theta_1^0)$. Let $\theta_2 \in \Theta$ be independent of $\theta_1$. 
    Then $\est{p} = \permtest(\firstparam,\secondparam)$ is uniformly distributed on $\{\frac{i}{T+1}\}_{i=1}^T$.
\end{theorem}
\begin{proof}
    The proof is identical to that of Theorem \ref{thm:main}.
\end{proof}

\section{Transformer Architecture and Notation}
\label{app:llamaarchitecture}

We consider models with the Llama Transformer architecture and define the notation henceforth, although this can easily be extended to other Transformer architectures. 

Following the definition of $\fmlp$ in Example \ref{example:glu-mlp} and the definition of a submodel in Definition \ref{defn:submodel}, we can define an abstraction of the full Llama language model architecture consisting of $L$ Transformer blocks sandwiched between an input and output layer. For the sequel, we will abuse notation in applying $\fmlp$ to multi-dimensional tensors by broadcasting along the last axis.
We use $d,n \in \N$ to respectively denote the model dimension and sequence length, where $\Thetalm = \Thetain \times \Thetablock^{\times L} \times \Thetaout$ with $\Thetablock$ denoting the parameter space of each Transformer block and $\Thetain,\Thetaout$ denoting the parameter spaces the input and output layers.
We decompose $\Thetablock = \Thetapre \times \Thetamlp$ and use $\frest : \Thetapre \times \R^{n \times d} \to \R^{n \times d}$ to denote all remaining parts of the Transformer besides the MLP. The inputs to $\frest$ are the input and output of the MLP, and the output of $\frest$ is fed directly to the MLP of the next layer. In particular, $\frest$ 
takes the input and output to the MLP of layer $i$, and first performs the residual connection following the MLP of layer $i$, then the self-attention and normalization components of layer $i+1$, and returns the input to the MLP of layer $i+1$.
We use $\fin : \mc{X} \times \Thetain \to \R^{n \times d}$ and $\fout : \R^{n \times d} \times \Thetablock^{(L)} \to \mc{Y}$ to respectively denote the input and output layers, i.e. the elements before the first MLP and after the last MLP.
Putting everything together gives the following definition of the model; we introduce the notation $\activation{\theta}{i}$ in the definition as a matter of convenience to track intermediate activations.
\begin{definition}\label{defn:lm}
    (GLU Transformer model)
    Let $\theta = (\thetain, \{ \thetablock^{(i)} \}_{i=1}^L, \thetaout) \in \Thetalm$ and $X \in \mc{X}$, with $\thetablock^{(i)} = (\thetapre^{(i)},\thetamlp^{(i)})$.
    Then $\flm(X;\theta) = \fout(\activation{\theta}{L} ; \thetaout)$ for $\activation{\theta}{0}=\fin(X;\thetain)$
    and
    \begin{align} \label{eqn:Hix}
        \activation{\theta}{i} = \frest(\activation{\theta}{i-1}, \fmlp(\activation{\theta}{i-1})).
    \end{align}
\end{definition}

For a Llama model, Table \ref{tab:transformer_table} describes the shapes of the model weight matrices for $i = 1, \dots, L$, for $V$ (vocab size), $d_\text{emb}$ (the hidden dimension), and $d_\text{mlp}$ (MLP hidden dimension). Following Definition \ref{defn:lm}, we have $\thetain = (E), \thetablock^{(i)}= (\thetapre^{(i)}, \thetamlp^{(i)})$ where $\thetapre^{(i)} = (\gamma_{\text{input},i}, W_{Q,i}, W_{K,i}, W_{V,i}, W_{O,i}, \gamma_{\text{post-attn},i})$, $\thetamlp^{(i)} = (G_i, U_i, D_i)$, and $\thetaout = (\gamma_\text{final}, O)$. We now describe a forward pass of the model. 

\begin{table}
    \centering
    \begin{tabular}{c|c}
    \hline 
        \textbf{Parameter name} & \textbf{Notation} \\ \hline
        embedding & $E \in \R^{V \times d_{\text{emb}}}$ \\ \hline 
        input layernorm & $\gamma_{\text{input},i} \in \R^{1 \times d_{\text{emb}}}$ \\ 
        attention query matrix & $W_{Q,i} \in \R^{d_{\text{emb}} \times d_{\text{emb}}}$ \\ 
        attention key matrix & $W_{K,i} \in \R^{d_{\text{emb}} \times d_{\text{emb}}}$ \\ 
        attention value matrix & $W_{V,i} \in \R^{d_{\text{emb}} \times d_{\text{emb}}}$ \\ 
        attention output matrix & $W_{O,i} \in \R^{d_{\text{emb}} \times d_{\text{emb}}}$ \\ 
        post-attention layernorm & $\gamma_{\text{post-attn, } i} \in \R^{1 \times d_{\text{emb}}}$\\ \hline 
        MLP gate projection & $G_i \in \R^{d_{\text{mlp}} \times d_{\text{emb}}}$ \\ 
        MLP up projection & $U_i \in \R^{d_{\text{mlp}} \times d_{\text{emb}}}$ \\ 
        MLP down projection & $D_i \in \R^{d_{\text{emb}} \times d_{\text{mlp}}}$ \\ \hline 
        final layernorm & $\gamma_\text{final} \in \R^{1 \times d_{\text{emb}}}$ \\ 
        linear output & $O \in \R^{d_{\text{emb}} \times V}$ \\ \hline 
    \end{tabular}
    \caption{We describe our notation and the dimensions of parameters of the Llama model architecture. Here, $i$ ranges over the number of Transformer blocks.}
    \label{tab:transformer_table}
\end{table}

We define the softmax function on a vector $v = (v_1, \dots, v_n)$, softmax$(v)$, as
\begin{equation*}
    \text{softmax}(v)_i = \frac{e^{v_i}}{\sum_{k=1}^n e^{v_k}}.
\end{equation*}
On batched input $X \in \R^{N \times n \times m}$ where each $X^{(b)} = [w_1 | \dots | w_m] \in \R^{n\times m}$ with column vectors $w_i$, we define the softmax as 
\begin{equation*}
    \text{softmax}(X^{(b)}) = [ \text{softmax}(w_1)| \dots | \text{softmax}(w_m) ],
\end{equation*}
\begin{equation*}
    \text{softmax}(X) = [ \text{softmax}(X^{(1)})| \dots | \text{softmax}(X^{(N)}) ].
\end{equation*}

For a forward pass of the model $\flm(X;\theta)$, consider an input sequence of tokens $X \in \{ 0, 1 \}^{N \times V}$ as one-hot vectors where $n$ is sequence length. Then 

We feed the input through: 
\begin{enumerate}
     \item ($\fin$) Embedding layer: 
    \begin{equation*}
        \activation{\theta}{0}= \fin(X; \thetain) = X E \in \R^{N \times d_\text{emb}}
    \end{equation*}
    \item ($\fpre, \fmlp, \fpost$) For each Transformer block $i = 0, 1, \dots, L$, through $\fpre$, $\fmlp$, and $\fpost$:
    \begin{enumerate}
        \item Input layernorm: 
        \begin{equation*}
            X_{\text{LN}_1}^{(i)} = \frac{\activation{\theta}{i}}{\sqrt{\text{Var}(\activation{\theta}{i}) + \varepsilon}} \odot \gamma_{\text{input}, i}
        \end{equation*}
        (with variance over the last axis) for some offset $\varepsilon$ (typically 1e-6). 
        \item Causal multi-head self-attention: Split $X_{\text{LN}_1}^{(i)}$ on the first axis into nheads $X_{\text{LN}_1, j}^{(i)}, \dots, X_{\text{LN}_1, \text{nheads}}^{(i)}$. On each head $X_{\text{LN}_1, j}^{(i)}$, 
        \begin{equation*}
            X_{\text{SA}, j}^{(i)} = \text{self-attn}(X_{\text{LN}_1, j}^{(i)}) = \text{softmax}\left( \frac{X_{\text{LN}_1, j}^{(i)} W_{Q,i}^T (X_{\text{LN}_1, j}^{(i)} W_{K,i}^T)^T}{\sqrt{d_\text{emb}}} \right) X_{\text{LN}_1, j}^{(i)} W_{V,i}^T W_{O,i}^T
        \end{equation*}
        and concatenate $X_{\text{SA}, j}^{(i)}$ along the first axis again as $X_{\text{SA}}^{(i)}$.
        \item Dropout and residual connection: $X_{\text{DR}_1}^{(i)} = X_{\text{LN}_1}^{(i)} + \text{Dropout}(X_{\text{SA}}^{(i)})$
        \item Post-attention layernorm: 
        \begin{equation*}
            X_{\text{LN}_2}^{(i)} = \frac{X_{\text{DR}_1}^{(i)}}{\sqrt{\text{Var}(X_{\text{DR}_1}^{(i)}) + \varepsilon}} \odot \gamma_{\text{post-attn}, i}
        \end{equation*}
        (with variance over the last axis) for some offset $\varepsilon$. Then we have 
        \begin{equation*}
            \fpre(\activation{\theta}{i-1}; \thetapre^{(i)})=X_{\text{LN}_2}^{(i)}. 
        \end{equation*}
        \item Next, we feed through $\fmlp$, the multi-layer perceptron:
        \begin{equation*}
            \fmlp(X_{\text{LN}_2}^{(i)}; \thetamlp^{(i)}) = X_i^{\text{MLP}} = [ \sigma(X_i^{\text{LN}_2} G_i^T) \odot (X_i^{\text{LN}_2} U_i^T) ] D_i^T
        \end{equation*}
        for some activation $\sigma$ (e.g., SiLU). 
        \item Finally, we feed through $\fpost$, dropout and the residual connection: 
        \begin{equation*}
        \fpost(\thetamlp^{(i)})= \activation{\theta}{i+1} = X_i^{\text{DR}_1} + \text{Dropout}(X_i^{\text{MLP}})
        \end{equation*}
    \end{enumerate}
    \item ($\fout$) Final layernorm on the output $\activation{\theta}{N+1}$ from the final Transformer block:
    \begin{equation*}
        X_{\text{LN}}^{(L)} = \frac{\activation{\theta}{L}}{\sqrt{\text{Var}(\activation{\theta}{L}) + \varepsilon}} \odot \gamma_\text{final}
    \end{equation*}
    (with variance over the last axis) for some offset $\varepsilon$. Then, linear output embedding and softmax mapping to output probabilities: 
    \begin{equation*}
        \fout(\activation{\theta}{L}) = \text{softmax}(X_{\text{LN}}^{(L)}O^T), 
    \end{equation*}
\end{enumerate}
which defines the entire forward pass $\flm(X;\theta)$.

\section{Model Transformation Class}
\label{app:llama_permutation}

We describe two sets of equivariant transformations $\Pi$ on a Transformer model as described in Appendix \ref{app:llamaarchitecture}. (Abusing notation), the first set, $\Pi_\text{emb}$, consists of elements ${\pi_\text{emb}}$ where $\pi_\text{emb} \in \R^{d_\text{emb} \times d_\text{emb}}$ is a permutation matrix. The second set, $\Pi_\text{mlp}$, consists of elements $\pi_{\text{mlp}}$ where $\pi_\text{mlp} \in \R^{d_\text{mlp} \times d_\text{mlp}}$ is a permutation matrix. 

\begin{enumerate}
    \item $\pi_\text{emb}(\theta)$: Applying an embedding permutation $\pi_\text{emb} \in \mathbb{R}^{d_\text{emb} \times d_\text{emb}}$ by left or right multiplying all relevant matrices by $\xi_\text{embed}$ (permuting rows or columns).
    \item $\pi_\text{mlp}(\theta)$: Applying MLP permutations $\pi_{\text{mlp},i} \in \mathbb{R}^{d_\text{mlp} \times d_\text{mlp}}$ to MLP layers.
\end{enumerate}
These permutations are applied such that the outputs of the original model $\theta$ and the permuted model $\pi(\theta)$ remain aligned. We describe the details in Table \ref{tab:permute_model}. 
\begin{table}
    \centering
    \begin{tabular}{c|c|c|c}
    \hline 
        \textbf{Parameter name} & $\theta$ & $\pi_\text{emb}(\theta)$ & $\pi_\text{mlp}(\theta)$ \\ \hline
        embedding & $E$ & $E \pi_\text{emb}$ & $E$ \\ \hline 
        input layernorm & $\gamma_{\text{input},i}$ & $\gamma_{\text{input},i} \pi_\text{emb}$ & $\gamma_{\text{input},i}$ \\ 
        attention query matrix & \textbf{$W_{Q,i}$} & \textbf{$W_{Q,i}$}$\pi_\text{emb}$ & \textbf{$W_{Q,i}$} \\ 
        attention key matrix & \textbf{$W_{K,i}$} & \textbf{$W_{K,i}$}$\pi_\text{emb}$ & \textbf{$W_{K,i}$} \\ 
        attention value matrix & \textbf{$W_{V,i}$} & \textbf{$W_{V,i}$}$\pi_\text{emb}$ & \textbf{$W_{V,i}$} \\ 
        attention output matrix & \textbf{$W_{O,i}$} & $\pi_\text{emb}^T$\textbf{$W_{O,i}$} & \textbf{$W_{O,i}$} \\  
        post-attention layernorm & $\gamma_{\text{post-attn, } i}$ & $\gamma_{\text{post-attn, } i} \pi_\text{emb}$ & $\gamma_{\text{post-attn, } i}$ \\ \hline
        MLP gate projection & $G_i$ & $G_i \pi_\text{emb}$ & $\pi_{\text{mlp}, i} G_i$ \\ 
        MLP up projection & $U_i$ & $U_i \pi_\text{emb}$ & $\pi_{\text{mlp}, i} U_i$ \\ 
        MLP down projection & $D_i$ & $\pi_\text{emb}^T D_i$ & $D_i \pi_{\text{mlp}, i}^T$ \\ \hline 
        final layernorm & $\gamma_\text{final}$ & $\gamma_\text{final} \pi_\text{emb}$ & $\gamma_\text{final}$ \\ 
        linear output & $O$ & $\pi_\text{emb}^T O$ & $O$ \\ \hline 
    \end{tabular}
    \caption{We describe the equivariant transformation classes $\pi_\text{emb}$ and $\pi_\text{mlp}$ that we use in Llama-architecture model experiments. When running $\permtest$, we compose both $\pi_\text{emb}$ and $\pi_\text{mlp}$, with random permutations $\pi$ selected.}
    \label{tab:permute_model}
\end{table}

\section{Additional Constrained Setting Experimental Results}
\label{app:exp_results}
\label{app:nonrobust_results}

We report p-values from the constrained setting experiments --- statistics $\ltwo, \csw$, and $\csh$ on all 210 model pairs (from 21 Llama 2-architecture models) in Figures \ref{fig:ell2_pvalues}, \ref{fig:csw_pvalues}, and \ref{fig:csh_pvalues}, where the model names are colored by base model (ground truth). For all statistics, the p-values on independent model pairs are uniformly distributed, while they are all significant at $0.01$ and smaller for $\csw$ and $\csh$ (at $\varepsilon = $ 2.2e-308) for fine-tuned model pairs. 

As described in Section \ref{sec:constrained-results}, we compute $\ltwo$ with $\permtest$ with $T=99$; and for $\csw$ and $\csh$ we aggregate p-values across the 32 MLP submodels with $\fisher$. For $\csh$, activations are generated using inputs of sequences of tokens sampled uniformly from the models' vocabulary.

\begin{figure}
    \centering
    \includegraphics[width=\linewidth]{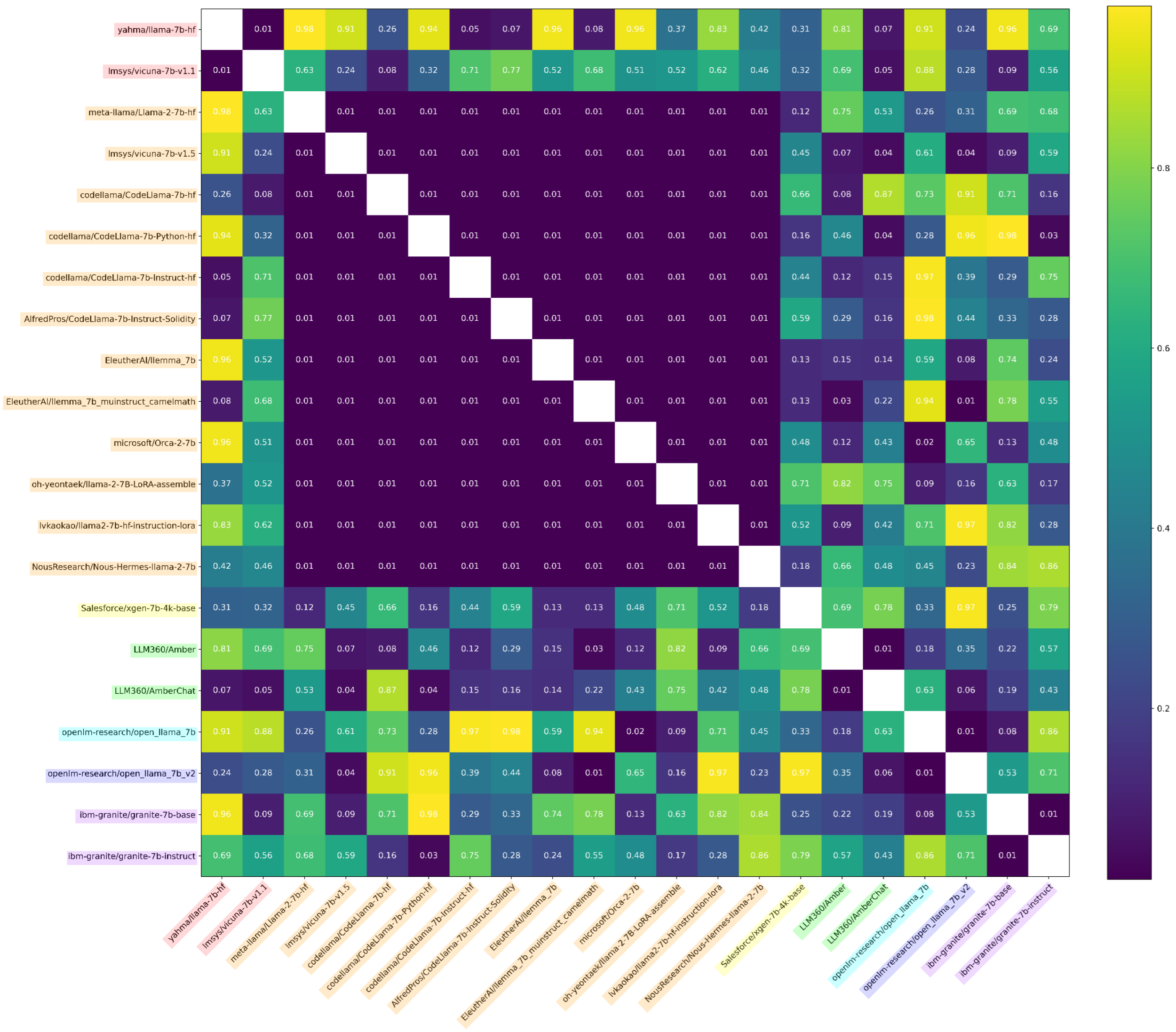}
    \caption{We report p-values from $\ltwo$ on Llama-7B model pairs, where $\ltwo$ is computed with $\permtest$ and $T=99$.}
    \label{fig:ell2_pvalues}
\end{figure}

\begin{figure}
    \centering
    \includegraphics[width=\linewidth]{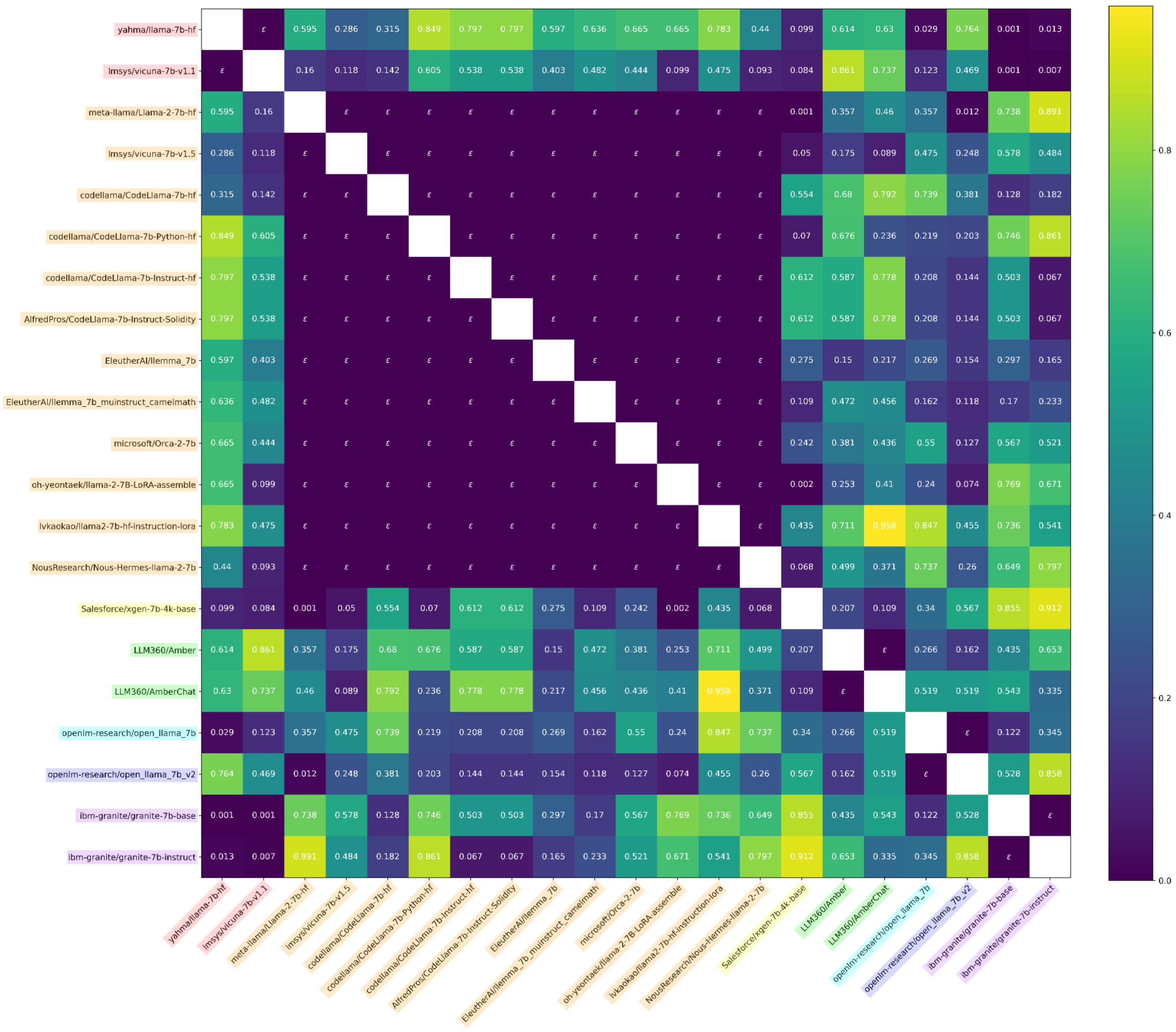}
    \caption{We report p-values from $\csw$ on Llama-7B model pairs, where $\varepsilon =$ 2.2e-308. We use $\speartest$ on $\csw$ and aggregate with $\fisher$ across the 32 MLPs.}
    \label{fig:csw_pvalues}
\end{figure}

\begin{figure}
    \centering
    \includegraphics[width=\linewidth]{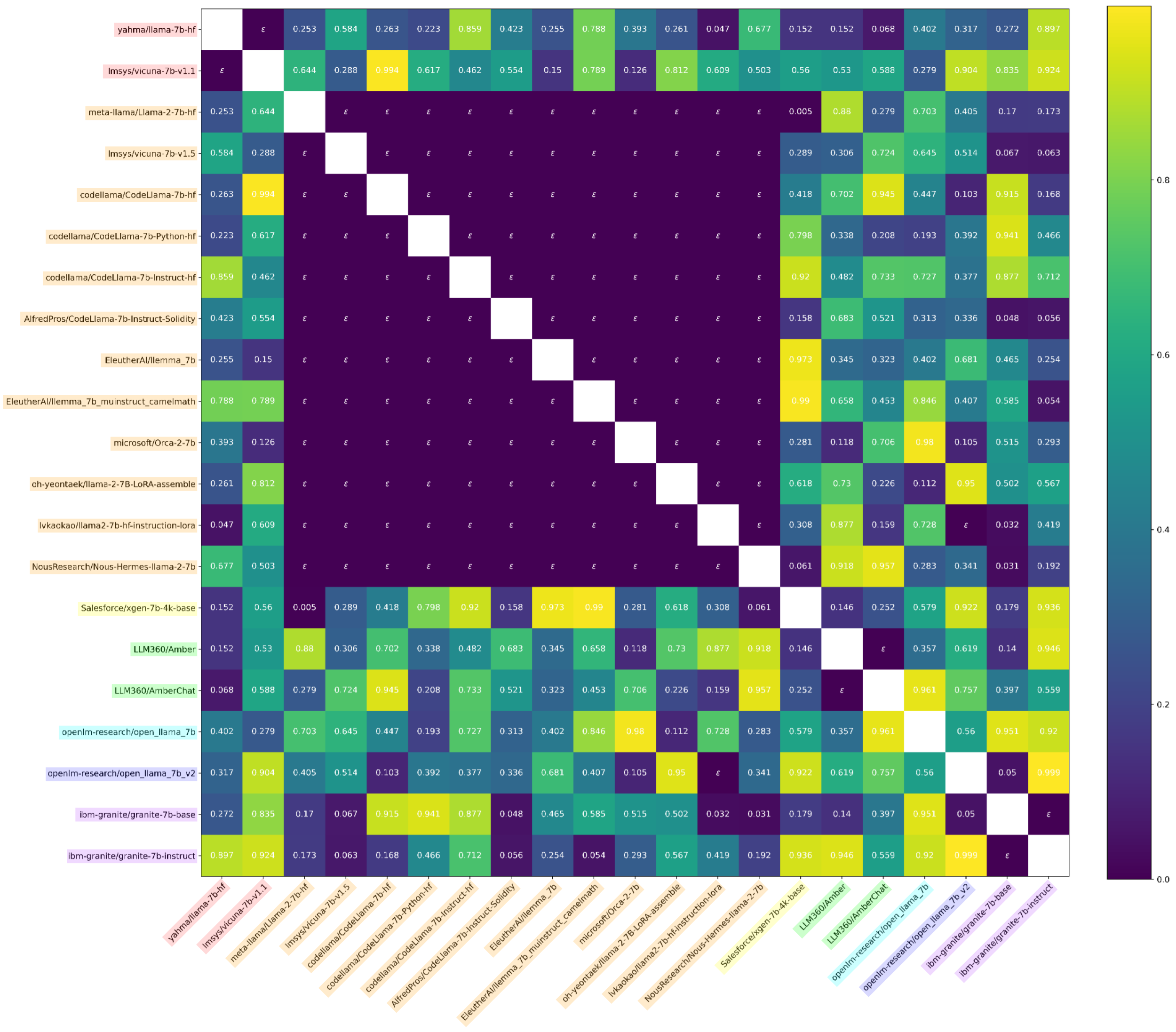}
    \caption{We report p-values from $\csh$ on Llama-7B model pairs, where $\varepsilon =$ 2.2e-308. We use $\speartest$ on $\csh$ and aggregate with $\fisher$ across the 32 MLPs.}
    \label{fig:csh_pvalues}
\end{figure}

\section{Additional Unconstrained Setting Experimental Results}
\label{app:robust_results}

We report values of $\rob$ on all model pairs in Figure \ref{fig:robust_stat}. The statistic is low ($< \varepsilon =$ 2.2e-308) for all non-independent model pairs, and uniformly distributed for independent model pairs, empirically acting as a p-value. 

\begin{figure}
    \centering
    \includegraphics[width=\linewidth]{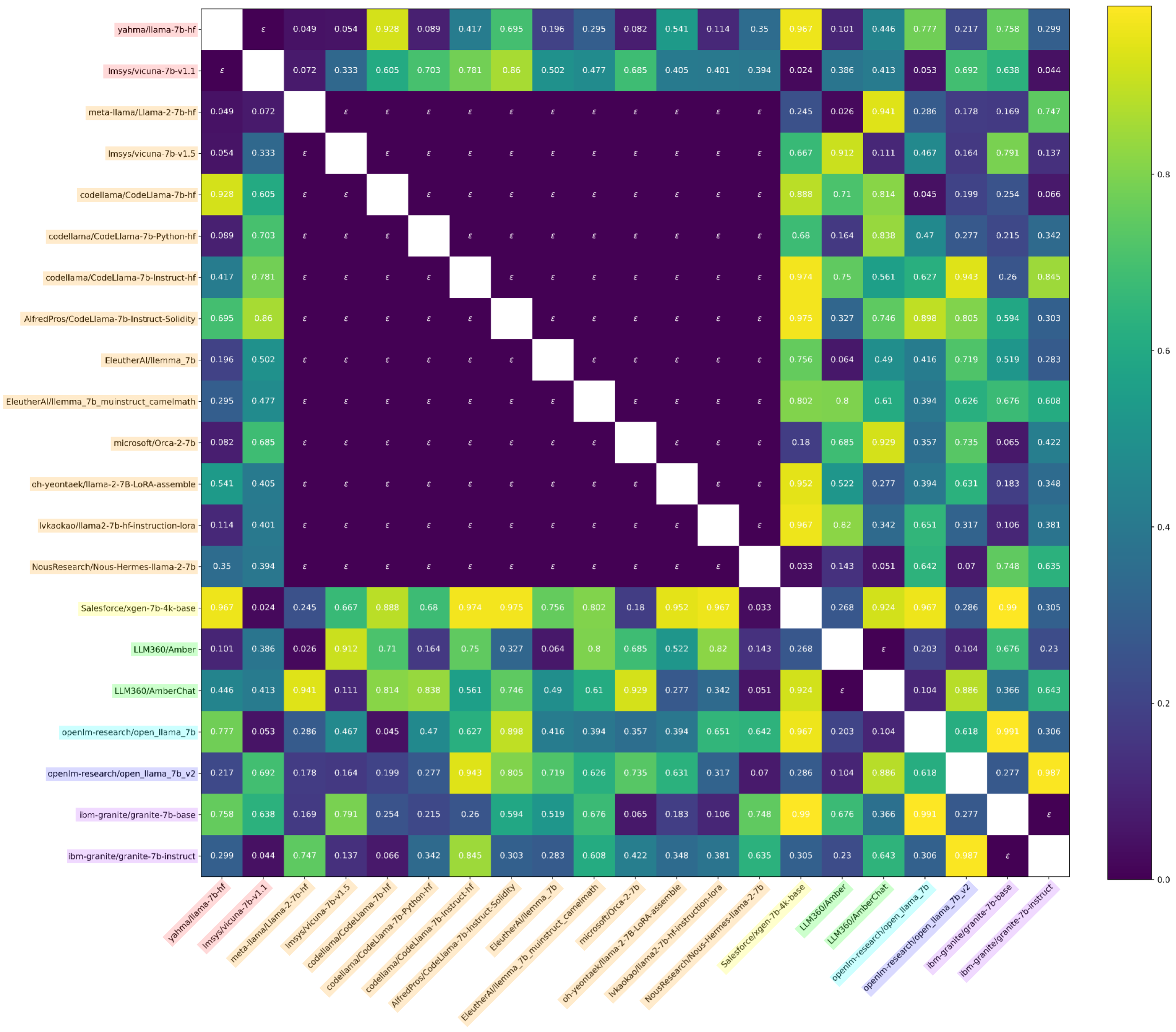}
    \caption{We report values of $\rob$ on Llama-7B model pairs, where $\varepsilon =$ 2.2e-308. We use $\speartest$ on $\rob$ and aggregate with $\fisher$ across the 32 MLPs.}
    \label{fig:robust_stat}
\end{figure}

\subsection{Striped Hyena Experiments}
\label{app:stripedhyena}
We report $\csw$ on specific parameters from \texttt{StripedHyena-Nous-7B} and \texttt{Mistral-7B-v0.1} shown in Table \ref{tab:stripedhyena}. We no longer only evaluate $\csw$ on MLP up projection matrices, so that we can investigate similarity in other parameters as well. These p-values no longer satisfy the independence requirement of Theorem \ref{thm:fisher}, so we do not aggregate them with $\fisher$.
\begin{table}[]
    \centering
    \begin{tabular}{c|c|c}
    \hline 
        Parameter name & Notation & $\csw$ \\ \hline 
        embedding & $E$ & 1.61e-16 \\ 
        attention query matrix & $W_Q^{(1)}$ & 6.17e-190 \\ 
        attention key matrix & $W_K^{(1)}$ & 1.47e-7 \\ 
        attention value matrix & $W_V^{(1)}$ & 1.56e-114 \\ 
        attention query matrix & $W_Q^{(1)}$ & 6.17e-190 \\ 
        attention output matrix & $W_O^{(1)}$ & 0.010 \\ 
        MLP gate projection & $G^{(1)}$ & 0.517 \\ 
        MLP up projection & $U^{(1)}$ & 0.716 \\ 
        MLP down projection & $D^{(1)}$ & 6.03e-80 \\ \hline 
    \end{tabular}
    \caption{We report $\csw$ on embedding and Transformer Block 1 parameters from \texttt{StripedHyena-Nous-7B} and \texttt{Mistral-7B-v0.1}, using $\csw$ on the specific parameters rather than only up-projection matrices. We find low p-values between some layers, including the embedding  matrix and attention matrices.}
    \label{tab:stripedhyena}
\end{table}

\subsection{MLP Retraining Experiments}
\label{app:retrain}

We retrain each of the 32 MLP layers by feeding in random inputs through the original MLP (gate, up, and down projection matrices.) We train for 10000 gradient steps using MSE loss and an Adam Optimizer with a learning rate of 0.001 and batch size of 5000. A sample learning curve is in Figure \ref{fig:mlpretrain_curve}.

\begin{figure}[h]
    \centering
    \includegraphics[width=0.5\linewidth]{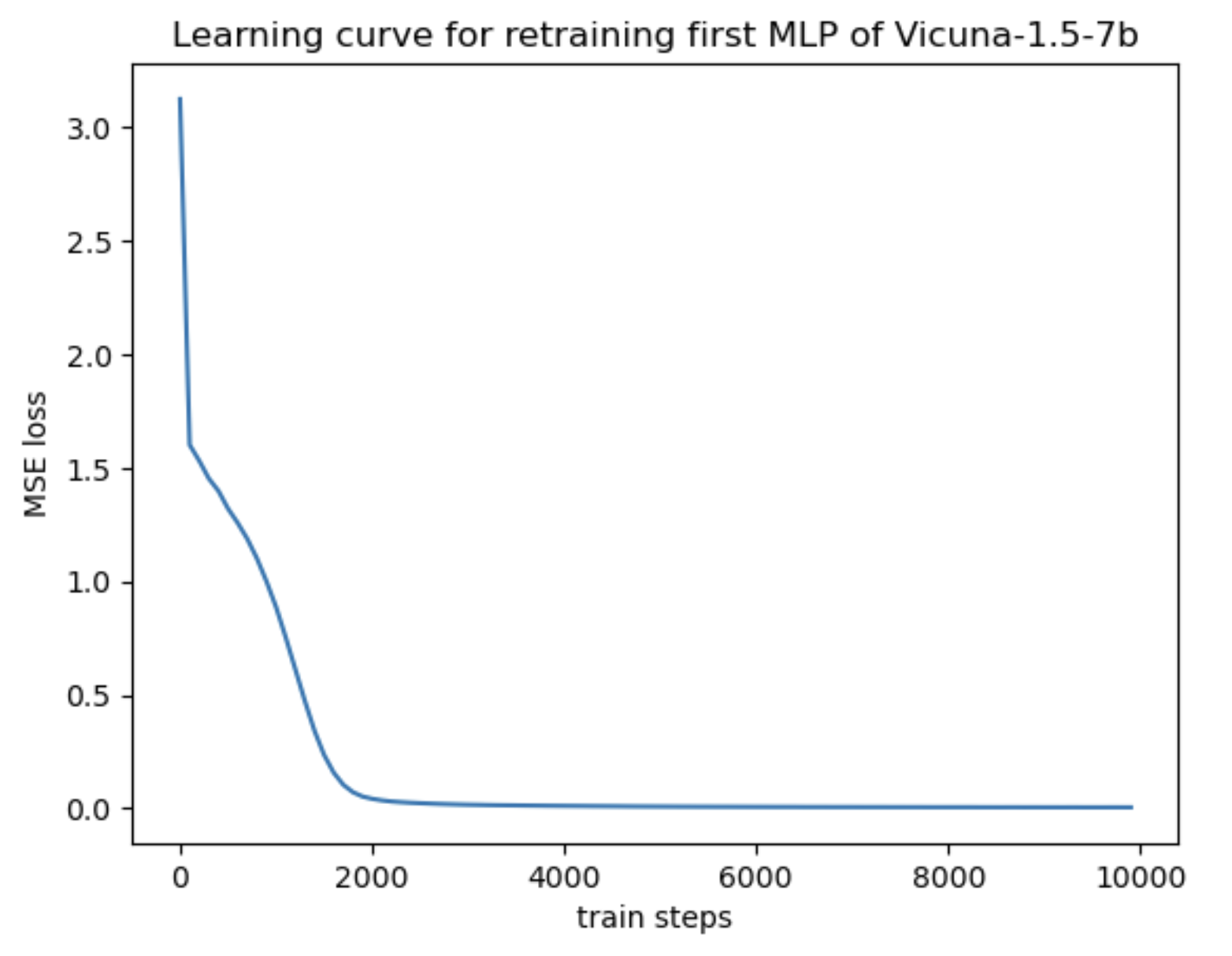}
    \caption{We show a learning curve for the MLP retraining experiments used in Section \ref{sec:mlp-retrain} (retraining one MLP).}
    \label{fig:mlpretrain_curve}
\end{figure}

The MLP retraining results for all 32 MLP layers of \texttt{vicuna-7b-v1.5}, compared with \texttt{Llama-2-7b-hf} are in Table \ref{tab:mlp_retrain}, showing that the statistic is robust to retraining of all layers. 

\begin{table}[]
  \begin{minipage}{0.32\linewidth}
    \centering
    \begin{tabular}{c|c|c|}
    \hline 
        MLP & Loss & $\log_{10}(\rob^{(i)})$ \\ \hline
        1 & 0.0048 & ${-479}$ \\ 
        2 & 0.012 & ${-485}$ \\ 
        3 & 0.0026 & ${-614}$ \\ 
        4 & 0.0034 & ${-580}$ \\ 
        5 & 0.0030 & ${-523}$ \\ 
        6 & 0.0035 & ${-513}$ \\ 
        7 & 0.0041 & ${-533}$ \\ 
        8 & 0.0042 & ${-464}$ \\ 
        9 & 0.0050 & ${-439}$ \\ 
        10 & 0.0050 & ${-377}$ \\ 
        11 & 0.0060 & ${-365}$ \\ \hline 
    \end{tabular}
  \end{minipage}
  \hfill
  \begin{minipage}{0.32\linewidth}
    \centering
\begin{tabular}{c|c|c|}
\hline 
        MLP & Loss & $\log_{10}(\rob^{(i)})$ \\ \hline
        12 & 0.0060 & ${-342}$ \\ 
        13 & 0.0058 & ${-330}$ \\ 
        14 & 0.0066 & ${-323}$ \\ 
        15 & 0.0063 & ${-414}$ \\ 
        16 & 0.0061 & ${-394}$ \\ 
        17 & 0.0063 & ${-445}$ \\ 
        18 & 0.0055 & ${-515}$ \\ 
        19 & 0.0045 & ${-571}$ \\ 
        20 & 0.0045 & ${-512}$ \\ 
        21 & 0.0047 & ${-595}$ \\ 
        22 & 0.0043 & ${-555}$ \\ \hline 
    \end{tabular}
  \end{minipage}
    \hfill
  \begin{minipage}{0.32\linewidth}
    \centering
\begin{tabular}{c|c|c}
\hline 
        MLP & Loss & $\log_{10}(\rob^{(i)})$ \\ \hline
        23 & 0.0043 & ${-593}$ \\ 
        24 & 0.0047 & ${-542}$ \\ 
        25 & 0.0050 & ${-497}$ \\ 
        26 & 0.0051 & ${-534}$ \\
        27 & 0.0052 & ${-482}$ \\ 
        28 & 0.0061 & ${-477}$ \\ 
        29 & 0.0065 & ${-433}$ \\
        30 & 0.0098 & ${-361}$ \\ 
        31 & 2.313 & ${-26.4}$ \\ 
        32 & 0.0114 & ${-174}$ \\ 
         & & \\ \hline 
    \end{tabular}
  \end{minipage}
  \caption{We retrain individual MLP blocks of \texttt{Llama-2-7b-hf} and \texttt{vicuna-7b-v1.5} then compute $\rob^{(i)}$ using Algorithm \ref{algorithm:general-test}. Even after retraining MLP layers, $\rob$ remains small (with log values reported here).
    }
    \label{tab:mlp_retrain}
\end{table}

\subsection{Fine-grained Forensics and Localized Testing}
\label{app:modelblockmatching}

As described in Section \ref{sec:localized-testing}, we can run $\rob$ on all pairs of Transformer blocks between two models (of different architecture), as long as they share the GLU structure. In addition to the Llama 3 results, we report results of matched blocks on the Sheared-LLaMa and Nvidia-Minitron models, which are both pruned from Llama models. 

In particular, we were able to identify the specific Transformer blocks of $\theta_{8B}= \texttt{Llama-3.1-8B}$ whose weights were likely used in initializing $\theta_{3B}=\texttt{Llama-3.2-3B}$ and $\theta_{1B}=\texttt{Llama-3.2-1B}$, as Meta reported that the \texttt{Llama-3.2-3B} and \texttt{Llama-3.2-1B} models were pruned from \texttt{Llama-3.1-8B} \citep{llamablog}. We use $\rob$ on all pairs of MLP blocks, where $(d_{\theta_{8B}},h_{\theta_{8B}},N_{\theta_{8B}})=(4096,14336,32)$,$(d_{\theta_{3B}},h_{\theta_{3B}},N_{\theta_{3B}})=(3072,8192,28)$, and $(d_{\theta_{1B}},h_{\theta_{1B}},N_{\theta_{1B}})=(2048,8192,16)$. We match blocks when the statistic $\rob^{(i,j)}$ from block $i$ of model 1 and block $j$ of model 2 is less than 1e-4, reported in Tables \ref{tab:llama3b} and \ref{tab:llama1b} (with the same for the other matchings in this section). 

\begin{table}[H]
    \centering
    \small 
    \begin{tabular}{c|c|c|c|c|c|c|c|c|c|c|c|c|c|c|c|c}
    \hline 
        $i$ & 1 & 2 & 3 & 4 & 5 & 6 & 7 & 8 & 9 & 10 & 11 & 12 & 13 & 14 & 15 & 16 \\ \hline
        $j:\rob^{(i,j)}(\theta_{8B},\theta_{3B})<1$e-4 & 1 & 2 & 3 & 4 & 5 & 6 & 7 & 8 & 9 & 10 & & 11 & 12 & 13 & 14 & 15 \\ \hline
        \multicolumn{17}{c}{} \\ \hline 
        $i$ & 17 & 18 & 19 & 20 & 21 & 22 & 23 & 24 & 25 & 26 & 27 & 28 & 29 & 30 & 31 & 32 \\ \hline
        $j:\rob^{(i,j)}(\theta_{8B},\theta_{3B})<1$e-4 &  & 16 & 17 & 18 & 19 & 20 &  & 21 & 22 & 23 & 24 & 25 &  & 26 & 27 & 28 \\ \hline 
    \end{tabular}
    \caption{We compute $\rob^{(i,j)}$ on all MLP blocks $i \in \{1,\dots,32 \}$ of $\theta_{8B}= \texttt{Llama-3.1-8B}$ and MLP blocks $j \in \{1,\dots, 28 \}$ of  $\theta_{3B}=\texttt{Llama-3.2-3B}$. We report pairs $(i,j)$ such that $\rob^{(i,j)} <1$e-4.  We find our test can identify the matching blocks even after pruning.}
    \label{tab:llama3b}
\end{table}

\begin{table}[H]
    \centering
    \small 
    \begin{tabular}{c|c|c|c|c|c|c|c|c|c|c|c|c|c|c|c|c}
    \hline 
        $i$ & 1 & 2 & 3 & 4 & 5 & 6 & 7 & 8 & 9 & 10 & 11 & 12 & 13 & 14 & 15 & 16 \\ \hline
        $j:\rob^{(i,j)}(\theta_{8B},\theta_{1B})<1$e-4 & 1 & 2 & 3 & 4 & 5 & 6 & & & 7 & & & 8 & & & 9 & \\ \hline 
        \multicolumn{17}{c}{} \\ \hline 
        $i$ & 17 & 18 & 19 & 20 & 21 & 22 & 23 & 24 & 25 & 26 & 27 & 28 & 29 & 30 & 31 & 32 \\ \hline
        $j:\rob^{(i,j)}(\theta_{8B},\theta_{1B})<1$e-4 & & 10 & & & 11 & & & & & & & & & & 15 & 16 \\ \hline 
    \end{tabular}
    \caption{We compute $\rob^{(i,j)}$ on all MLP blocks $i \in \{1,\dots,32 \}$ of $\theta_{8B}= \texttt{Llama-3.1-8B}$ and MLP blocks $j \in \{1,\dots, 16 \}$ of  $\theta_{1B}=\texttt{Llama-3.2-1B}$. We report pairs $(i,j)$ such that $\rob^{(i,j)} <1$e-4.}
    \label{tab:llama1b}
\end{table}

Next, we test \texttt{Sheared-LLaMa-2.7B}, with 32 Transformer blocks, hidden dimension 2560 and MLP dimension 6912. All 32 blocks align with the 32 blocks of Llama 2-7B, although both hidden and MLP dimensions have been reduced through pruning. 
\begin{table}[H]
    \centering
    \small 
    \begin{tabular}{c|c|c|c|c|c|c|c|c|c|c|c|c|c|c|c|c}
    \hline 
        $i$ & 1 & 2 & 3 & 4 & 5 & 6 & 7 & 8 & 9 & 10 & 11 & 12 & 13 & 14 & 15 & 16 \\ \hline
        $j:\rob^{(i,j)}(\theta_{1},\theta_{2})<1$e-90 & 1 & 2 & 3 & 4 & 5 & 6 & 7 & 8 & 9 & 10 & 11 & 12 & 13 & 14 & 15 & 16 \\ \hline 
        \multicolumn{17}{c}{} \\ \hline 
        $i$ & 17 & 18 & 19 & 20 & 21 & 22 & 23 & 24 & 25 & 26 & 27 & 28 & 29 & 30 & 31 & 32 \\ \hline
        $j:\rob^{(i,j)}(\theta_{1},\theta_{2})<1$e-90 &  17 & 18 & 19 & 20 &  21 & 22 & 23 & 24 & 25 & 26 & 27 & 28 & 29 & 30 & 31 & 32 \\ \hline 
    \end{tabular}
    \caption{We compute $\rob^{(i,j)}$ on all MLP blocks $i \in \{1,\dots, 32\}$ of $\theta_{1}= \texttt{Sheared-LLaMa-2.7B}$ and MLP blocks $j \in \{1,\dots, 32 \}$ of  $\theta_{2}=\texttt{Llama-2-7B}$. We report pairs $(i,j)$ such that $\rob^{(i,j)} <1$e-90.}
    \label{tab:shearedllama2-7b}
\end{table}

Next, we test \texttt{Sheared-LLaMa-1.3B}, with 24 Transformer blocks, hidden dimension 2048 and MLP dimension 5504.

\begin{table}[H]
    \centering
    \begin{tabular}{c|c|c|c|c|c|c|c|c|c|c|c|c}
    \hline 
        $i$ & 1 & 2 & 3 & 4 & 5 & 6 & 7 & 8 & 9 & 10 & 11 & 12 \\ \hline
        $j:\rob^{(i,j)}(\theta_{1},\theta_{2})<1$e-5 & 1 & 2 & 3 & 4 & 5 & 6 & 7 & 8 & 10 & 12 & & 16   \\ \hline 
        \multicolumn{13}{c}{} \\ \hline 
        $i$ & 13 & 14 & 15 & 16 & 17 & 18 & 19 & 20 & 21 & 22 & 23 & 24 \\ \hline
        $j:\rob^{(i,j)}(\theta_{1},\theta_{2})<1$e-5 & 17 & 18 & 19 & 20 & 21 & 22 & 25 & 27 & 28 & 29 & 31 & 32 \\ \hline 
    \end{tabular}
    \caption{We compute $\rob^{(i,j)}$ on all MLP blocks $i \in \{1,\dots, 24\}$ of $\theta_{1}= \texttt{Sheared-LLaMa-1.3B}$ and MLP blocks $j \in \{1,\dots, 32 \}$ of  $\theta_{2}=\texttt{Llama-2-7B}$. We report pairs $(i,j)$ such that $\rob^{(i,j)} <1$e-5.}
    \label{tab:shearedllama1-3b}
\end{table}

Finally, we compare Llama 3.1-8B with $\texttt{Llama-3.1-Minitron-4B-Depth-Base}$, a pruned model by reducing from 32 to 16 Transformer blocks and are able to identify the likely shared blocks. 
\begin{table}[H]
    \centering
    \begin{tabular}{c|c|c|c|c|c|c|c|c|c|c|c|c|c|c|c|c}
    \hline 
        $i$ & 1 & 2 & 3 & 4 & 5 & 6 & 7 & 8 & 9 & 10 & 11 & 12 & 13 & 14 & 15 & 16 \\ \hline
        $j:\rob^{(i,j)}(\theta_{1},\theta_{2})<1$e-90 & 1 & 2 & 3 & 4 & 5 & 6 & 7 & 8 & 9 & 10 & 11 & 12 & 13 & 14 & 15 & 32 \\ \hline 
    \end{tabular}
    \caption{We compute $\rob^{(i,j)}$ on all MLP blocks $i \in \{1,\dots, 32\}$ of $\theta_{1}= \texttt{Llama-3.1-Minitron-4B-Depth-Base}$ and MLP blocks $j \in \{1,\dots, 32 \}$ of  $\theta_{2}=\texttt{Llama-3.1-8B}$. We report pairs $(i,j)$ such that $\rob^{(i,j)} <1$e-90.}
    \label{tab:nvidia-minitron}
\end{table}

\section{Output-Preserving Transformations}
\label{app:robust}

An adversary could apply a particular rotation scheme by multiplying weight matrices by an orthogonal rotation matrix $U$ that will also preserve outputs. We describe such a transformation which breaks the invariants proposed by \cite{zeng2024humanreadablefingerprintlargelanguage} by manipulating layernorms. While this list may not be exhaustive, the following six transformations (with the first two described previously) ``camouflage'' the language model while preserving outputs: 
\begin{enumerate}[
        label={T\arabic*.},
        ref={T\arabic*.}]
    \item Permuting the rows of the embedding matrix (and subsequent matrices due to residual connections) by a permutation $\xi_{\text{emb}} \in \R^{d_{\text{emb}} \times d_{\text{emb}}}$
    \item Permuting the MLP matrices ($N$ different permutations for each Transformer block) by permutations $\xi_1, \dots, \xi_{N} \in \R^{d_{\text{mlp}} \times d_{\text{mlp}}}$
    \item Rotating the embedding matrix (and subsequent matrices due to residual connections) by an orthogonal rotation matrix $R_{\text{emb}} \in \R^{d_{\text{emb}} \times d_{\text{emb}}}$
    \item Rotating the query and key attention matrices ($N$ different rotations for each Transformer block) by orthogonal rotation matrices $R_1, \dots, R_{N} \in \R^{d_{\text{emb}} \times d_{\text{emb}}}$
    \item Replacing all layernorms (input, post-attention, final) with vectors in $\R^{1\times d_{\text{emb}}}$ with non-zero elements
    \item Scaling the MLP matrices by a constant non-zero factor
\end{enumerate}

Consider a model $\theta$ of Llama architecture (Appendix \ref{app:llamaarchitecture}). Consider orthogonal matrices $R_{\text{emb}}, R_1, \dots R_{32}$ as described, as well as new layernorms $\gamma'_{\text{input}, 1}, \dots, \gamma'_{\text{input}, 32}, \gamma'_{\text{post-attn}, 1}, \dots, \gamma'_{\text{post-attn}, 32}$ in $\R^{1 \times d_{\text{emb}}}$ with non-zero elements. 
Finally, consider non-zero constants $c_1, \dots, c_{32}$, which we use to transform the layernorms. We apply the rotation with these parameters to $\theta$, to get a new ``rotated'' model, Rot$(\theta)$. We generalize the set of transformations above as applying Rot$(\theta)$ to a model $\theta$. 

We transform all the original matrices of $\theta$ as in Table \ref{tab:rotation_scheme} (for $i = 1, \dots, 32$). 
Note that the transformations T1 and T2 are elements of the classes $\pi_\text{emb}$ and $\pi_\text{mlp}$, respectively, and the remaining transformations T3 to T6 are described in Table \ref{tab:rotation_scheme}. Importantly, T5 is the transformation that \cite{zeng2024humanreadablefingerprintlargelanguage}'s invariants are not robust to; our unconstrained setting test $\rob$ is robust to all 6 transformations, which we show in Table \ref{tab:huref_invariants}.

\subsection{Breaking HuREF Invariants}
\label{app:breakhuref}

Only transformations T3 and T5 is required to break the invariants from \cite{zeng2024humanreadablefingerprintlargelanguage}. Their first invariant is $M_a =  E (W_{Q,i})^T W_{K,i}) E^T$ at layer $i$, and for $M'$ with an embedding matrix rotation $R_\text{emb}$ where the layernorms $\gamma_{\text{input},i}$ are replaced with $\gamma'_{\text{input},i}$, we have the invariant is
\begin{align*}
    M'_a &= (E R_\text{emb}) \left( \text{diag}( \frac{1}{\gamma'_{\text{input}, i}}) R_\text{emb}^T  \text{diag}(\gamma_{\text{input}, i}) W_{Q,i}^T \right) \left( W_{K,i} \text{diag}(\gamma_{\text{input}, i})  R_\text{emb} \text{diag}( \frac{1}{\gamma'_{\text{input}, i}}) \right) (R_\text{emb}^T E^T) 
\end{align*}
and in general $M_a \neq M'_a$ unless the layernorm weights are equal constants. Similarly, we have the second invariant is 
\begin{align*}
    M'_b &= (E R_\text{emb}) \left( \text{diag}( \frac{1}{\gamma'_{\text{input}, i}}) R_\text{emb}^T  \text{diag}(\gamma_{\text{input}, i}) W_{V,i}^T \right) \left( W_{O,i} \text{diag}(\gamma_{\text{input}, i})  R_\text{emb} \text{diag}( \frac{1}{\gamma'_{\text{input}, i}}) \right) (R_\text{emb}^T E) \\ &\neq E W_{V,i}^T W_{O,i} E^T = M_b
\end{align*}
in general, also failing due to the layernorm and rotation $R_\text{emb}$ (note that our notation for Transformers is different than theirs). Finally, assuming for invariant $M_f$ that $W_1$ and $W_2$ are the gate and down projection matrices of an MLP (this is not stated explicitly in the paper but can be inferred from experiments), the remaining invariants do not hold either: 
\begin{align*}
    M'_f &= (E R_\text{emb}) \left( \text{diag}( \frac{1}{\gamma'_{\text{input}, i}}) R_\text{emb}^T  \text{diag}(\gamma_{\text{input}, i}) G_i^T \right) \left( D_i \text{diag}(\gamma_{\text{input}, i})  R_\text{emb} \text{diag}( \frac{1}{\gamma'_{\text{input}, i}}) \right) (R_\text{emb}^T E) \\ &\neq E G_i^T D_i E^T = M_f,
\end{align*}
so all their proposed invariants can be bypassed by these two transformations. 

Empirically, we we show this by computing all the invariants between \texttt{Llama-2-7b-hf} and independently trained models and between \texttt{Llama-2-7b-hf} and rotated finetuned models (including \texttt{Llama-2-7b-hf} itself) in Table \ref{tab:huref_invariants}. We can see there is little distinction between the independent vs. non-independent model pairs. 

\begin{table}[h]
    \centering
    \small
    \begin{tabular}{c c | c c c | c c c c}
    \hline 
        $\theta_1 = \texttt{Llama-2-7b-hf}, \theta_2=$ & Indep.? & $M_a$ & $M_b$ & $M_f$ & $\rob$ & $\csw$ & $\csh$ & $\jsd$ \\ \hline 
        \texttt{vicuna-7b-v1.5}& \ding{55} & 1.0 & 0.9883 & 0.9922 & $\varepsilon$ & $\varepsilon$ & $\varepsilon$ & -10.874 \\
        \texttt{Nous-Hermes-llama-2-7b}& \ding{55} & 1.0 & 1.0 & 1.0 & $\varepsilon$ & $\varepsilon$ & $\varepsilon$ & -12.101 \\ \hline 
         \texttt{llama-7b-hf} & \ding{51} & 0.0884 & 0.0250 & 0.0400 & 0.049 & 0.595 & 0.253 & -11.102 \\
         \texttt{AmberChat} & \ding{51} & 0.1289 & -0.0093 & 0.0198 & 0.941 & 0.460 & 0.279 & -10.281 \\
         \texttt{Openllama-v1} & \ding{51} & 0.1084 & 0.0076 & 0.0057 & 0.286 & 0.357 & 0.703 & -8.381 \\ \hline 
         Rotated \texttt{Llama-2-7b-hf} & \ding{55} & 0.0767 & 0.0908 & 0.1011  & $\varepsilon$ & 0.517 & 0.323 & $-\infty$\\ 
         Rotated \texttt{vicuna-7b-v1.5} & \ding{55} & 0.1553 & 0.0933 & 0.0977 & $\varepsilon$ & 0.688 & 0.857 & -10.874 \\
         Rotated \texttt{Nous-Hermes-llama-2-7b} & \ding{55} & 0.0332 & 0.0718 & 0.1060 & $\varepsilon$ & 0.772 & 0.240 & -12.101 \\ \hline 
    \end{tabular}
    \caption{Results for the three invariants $M_a, M_b, M_c$ from \cite{zeng2024humanreadablefingerprintlargelanguage} between \texttt{Llama-2-7b-hf} and independent and non-independent models.}
    \label{tab:huref_invariants}
\end{table}

Notably, we see that $\rob$ remains $\varepsilon$ even for rotated dependent model pairs, so our test is robust to these transformations. 

\begin{table}
    \centering
    \begin{tabular}{c|c|c}
    \hline 
        \textbf{Parameter name} & $\theta$ & $\text{Rot}(\theta)=\theta'$ \\ \hline
        embedding & $E$ & $ER_\text{emb}$ \\ 
        input layernorm & $\gamma_{\text{input, } i}$ & $\gamma'_{\text{input, } i}$ \\ 
        attention query matrix & \textbf{$W_{Q,i}$} & $R_i$ \textbf{$W_{Q,i}$} diag($\gamma_{\text{input, } i}$) $R_\text{emb}$ diag( $\frac{1}{\gamma'_{\text{input, } i}}$) \\ 
        attention key matrix & \textbf{$W_{K,i}$} & $R_i$ \textbf{$W_{K,i}$} diag($\gamma_{\text{input, } i}$) $R_\text{emb}$ diag( $\frac{1}{\gamma'_{\text{input, } i}}$) \\ 
        attention value matrix & \textbf{$W_{V,i}$} & \textbf{$W_{V,i}$} diag($\gamma_{\text{input, } i}$) $R_\text{emb}$ diag( $\frac{1}{\gamma'_{\text{input, } i}}$) \\ 
        attention output matrix & \textbf{$W_{O,i}$} & $R_\text{emb}^T$ \textbf{$W_{O,i}$} \\ 
        post-attention layernorm & $\gamma_{\text{post-attn, } i}$ & $\gamma'_{\text{post-attn, } i}$ \\
        MLP gate projection & $G_i$ & $G_i$ diag($\gamma_{\text{post-attn}, i}$) $R_\text{emb}$ diag( $\frac{1}{\gamma'_{\text{post-attn},i}}$) \\ 
        MLP up projection & $U_i$ & $c_i U_i$ diag($\gamma_{\text{post-attn}, i}$) $R_\text{emb}$ diag( $\frac{1}{\gamma'_{\text{post-attn},i}}$) \\ 
        MLP down projection & $D_i$ & $\frac{1}{c_i} $ $R_\text{emb}^T$ $D_i$ \\ 
        final layernorm & $\gamma_\text{final}$ & $\gamma'_\text{final}$ \\ 
        linear output & $O$ & $O$ diag($\gamma_\text{final}$) $R_\text{emb}$ diag($\frac{1}{\gamma'_\text{final}}$) \\ \hline 
    \end{tabular}
    \caption{We describe the output-preserving rotation applied to the parameters of a Llama-architecture model.}
    \label{tab:rotation_scheme}
\end{table}

\subsection{Invariance of Outputs under Rotation}

These transformations are particularly important because they preserve outputs as we show in Theorem \ref{thm:output-preserving}, and hence generally can go undetected, though $\rob$ is robust to them.

\begin{theorem}
\label{thm:output-preserving}
For any input sequence $X \in \{ 0, 1 \}^{n \times V}$, the outputs of models $\theta$ and $\text{Rot}(\theta)=\theta'$ are aligned, i.e. $\flm(X;\theta) = \flm(X;\theta')$.
\end{theorem}

\begin{proof}

First, note that an element-wise product of two one-dimensional vectors is equivalent to multiplying by the diagonal matrix of the second vector, i.e. for $v, \gamma \in R^{1 \times m}$, 
\begin{equation*}
    v \ast \gamma = v \text{diag}(\gamma).
\end{equation*}
We use this in our layernorm calculations. 

Let the output from the unrotated embedding layer be $y = \fin(X, E)= EX$ (for $X \in \{ 0, 1 \}^{n \times V}$). Then the output from the rotated embedding layer is $ y' = \fin(X,E')=(ER_\text{emb})(x) = y R_\text{emb}$. Now consider Transformer block $i$ with input $y$ and the rotated Transformer block with input $y R_\text{emb}$. $y$ is passed into the input layernorm, which returns 
\begin{equation*}
    z = LN_i(y) = \frac{y}{\sqrt{\var(y) + \varepsilon}} \odot \gamma_{\text{input}, i} = \frac{y}{\sqrt{\var(y) + \varepsilon}} \text{diag}(\gamma_{\text{input}, i}). 
\end{equation*}
The rotated input layernorm on $y'$ returns

\begin{align*}
    z' = LN'_i(y') &= \frac{y'}{\sqrt{\var(y') + \varepsilon}} \odot \gamma'_{\text{input}, i}  = \frac{y R_\text{emb}}{\sqrt{\var(y R_\text{emb}) + \varepsilon}} \odot \gamma'_{\text{input}, i} \\ &= \frac{y }{\sqrt{\var(y) + \varepsilon}} R_\text{emb} \text{diag}(\gamma'_{\text{input}, i}) = z \text{ diag}( \frac{1}{\gamma_{\text{input}, i}}) R_\text{emb} \text{diag}(\gamma'_{\text{input}, i}), 
\end{align*}
which follows from $R_\text{emb}$ being orthogonal. Then we have the output from the unrotated self-attention is 
\begin{align*}
    w = \text{softmax} \left( \frac{zW_{Q,i}^T (zW_{K,i}^T)^T}{\sqrt{d_\text{key}}} \right) zW_{V,i}^T W_{O,i}^T,
\end{align*}
and the output from the rotated self-attention with input $z'$ is 
\begin{align*}
    &\text{softmax} \left( \frac{z' (R_iW_{Q,i} \text{diag}(\gamma_{\text{input, } i}) R_\text{emb} \text{diag}( \frac{1}{\gamma'_{\text{input, } i}}) )^T (z' (R_i W_{K,i} \text{diag}(\gamma_{\text{input, } i}) R_\text{emb} \text{diag}( \frac{1}{\gamma'_{\text{input, } i}}))^T )^T }{\sqrt{d_\text{key}}} \right) \\ &z' (W_{V,i} \text{diag}(\gamma_{\text{input, } i}) R_\text{emb} \text{diag}( \frac{1}{\gamma'_{\text{input, } i}}))^T (R_\text{emb}^T W_{O,i})^T \\ &= 
    \text{softmax} \left( \frac{z' \text{diag}( \frac{1}{\gamma'_{\text{input, } i}}) R_\text{emb}^T \text{diag}(\gamma_{\text{input, } i})  W_{Q,i}^T  R_i^T (z' \text{diag}( \frac{1}{\gamma'_{\text{input, } i}}) R_\text{emb}^T  \text{diag}(\gamma_{\text{input, } i})   W_{K,i}^T R_i^T )^T }{\sqrt{d_\text{key}}} \right) \\ &z' \text{diag}( \frac{1}{\gamma'_{\text{input, } i}}) R_\text{emb}^T \text{diag}(\gamma_{\text{input, } i})  W_{V,i}^T  W_{O,i}^T R_\text{emb} \\ 
    &=
    \text{softmax} \left( \frac{z' \text{diag}( \frac{1}{\gamma'_{\text{input, } i}}) R_\text{emb}^T \text{diag}(\gamma_{\text{input, } i})   W_{Q,i}^T W_{K,i} \text{diag}(\gamma_{\text{input, } i}) R_\text{emb} \text{diag}( \frac{1}{\gamma'_{\text{input, } i}})  (z')^T }{\sqrt{d_\text{key}}}\right) z W^T_{V,i} W^T_{O,i} R_\text{emb} \\ 
    &= \text{softmax} \left( \frac{z W_{Q,i} W^T_{K,i}  z^T}{\sqrt{d_\text{key}}} \right) z W^T_{V,i} W^T_{O,i} R_\text{emb} 
    \\ &= w R_\text{emb} = w'. 
\end{align*}
Then $y$ and $y'$ respectively from before the layernorm are added as residual connections as $v = y+w$ and $v' = y'+w' = v R_\text{emb}$. $v$ is passed into the post-attention layernorm, which returns 
\begin{equation*}
    u = LN_i(v) = \frac{v}{\sqrt{\var(v) + \varepsilon}} \odot \gamma_{\text{post-attn}, i} = \frac{v}{\sqrt{\var(v) + \varepsilon}} \text{diag}(\gamma_{\text{post-attn}, i}). 
\end{equation*}
Similar to the input layernorm, the rotated post-attention layernorm on $v'$ returns 
\begin{align*}
    u' = LN'_i(v') &= \frac{v'}{\sqrt{\var(v') + \varepsilon}} \odot \gamma'_{\text{post-attn}, i}  = \frac{v R_\text{emb}}{\sqrt{\var(v R_\text{emb}) + \varepsilon}} \odot \gamma'_{\text{post-attn}, i} \\ &= \frac{v}{\sqrt{\var(v) + \varepsilon}} R_\text{emb} \text{diag}(\gamma'_{\text{post-attn}, i}) = u \text{ diag}( \frac{1}{\gamma_{\text{post-attn}, i}}) R_\text{emb} \text{diag}(\gamma'_{\text{post-attn}, i}).
\end{align*}

Then the output from the unrotated MLP layer on $u$ is 
\begin{align*}
    t &= [ \sigma(u G_i^T) \odot (u U_i^T) ] D_i^T
\end{align*}
and the output from the rotated MLP on $u'$ is 
\begin{align*}
    t' &= [ \sigma(u' (G_i \text{diag}(\gamma_{\text{post-attn}, i}) R_\text{emb} \text{diag}( \frac{1}{\gamma'_{\text{post-attn},i}}))^T \odot (u' (c_i U_i \text{diag}(\gamma_{\text{post-attn}, i}) R_\text{emb} \text{diag}( \frac{1}{\gamma'_{\text{post-attn},i}}))^T) ] (\frac{1}{c_i} R_\text{emb}^T D_i)^T \\ 
    &= [ \sigma(u \text{ diag}( \frac{1}{\gamma_{\text{post-attn}, i}}) R_\text{emb} \text{diag}(\gamma'_{\text{post-attn}, i}) \text{diag}( \frac{1}{\gamma'_{\text{post-attn},i}})  R_\text{emb}^T  \text{diag}(\gamma_{\text{post-attn}, i}) G_i^T) \odot \\ & (c_i u \text{ diag}( \frac{1}{\gamma_{\text{post-attn}, i}}) R_\text{emb} \text{diag}(\gamma'_{\text{post-attn}, i}) \text{diag}( \frac{1}{\gamma'_{\text{post-attn},i}})  R_\text{emb}^T  \text{diag}(\gamma_{\text{post-attn}, i})) U_i^T ] \frac{1}{c_i} D_i^T R_\text{emb} \\ 
    &= [ c_i \sigma(u G_i^T) \odot (u U_i^T) ] \frac{1}{c_i} D_i^T R_\text{emb} = t R_\text{emb}. 
\end{align*}
Then the output from the self-attention is added as a residual connection, and the final output from the unrotated Transformer block is $s = t+v$, and the output from the rotated Transformer block is $s' = t' + v' = s R_\text{emb}$. 

Suppose $a$ is the output after all Transformer layers in $\theta$ and $a'$ is the output after all Transformer layers in $\theta'$. Then the outputs after the final layernorms are
\begin{equation*}
    b = \frac{v}{\sqrt{\var(a) + \varepsilon}} \text{diag}(\gamma_\text{final})
\end{equation*}
\begin{equation*}
    b' = b \text{ diag}( \frac{1}{\gamma_{\text{final}}})  R_\text{emb} \text{diag}(\gamma'_{\text{final}}), 
\end{equation*}
and the logits from the linear output layer are 
\begin{align*}
    b O^T &= b \text{ diag}( \frac{1}{\gamma_{\text{final}}})  R_\text{emb} \text{diag}(\gamma'_{\text{final}}) \text{diag}(\gamma_\text{final})R_\text{emb}^T  \text{diag}( \frac{1}{\gamma'_\text{final}}) O^T \\ &= b' (O')^T,
\end{align*}
which are the same for both models.
\end{proof}

We attempted to undo such a transformation that an adversary may apply by solving the least squares problem: We solve for a rotation $A$ that minimizes $\vert AX - Y \vert $ where $X$ is a weight matrix of the first model and $Y$ is the corresponding weight matrix of the second model. Although this will provide a potential rotation to undo this transformation, we find that this solution will also find a matrix $A$ that aligns two independent model pairs as well. This makes undo-ing the rotation this way unreliable. 
The same holds for $X$ and $Y$ that are activations over multiple inputs. 

\end{document}